\documentclass[letterpaper]{article} 
\usepackage{aaai24}
\usepackage{times}  
\usepackage{helvet}  
\usepackage{courier}  
\usepackage[hyphens]{url}  
\usepackage{graphicx} 
\usepackage{natbib}  
\usepackage{caption} 
\DeclareCaptionStyle{ruled}{labelfont=normalfont,labelsep=colon,strut=off} 
\usepackage{algorithm}

\usepackage{newfloat}
\usepackage{listings}
\floatstyle{ruled}
\newfloat{listing}{tb}{lst}{}
\floatname{listing}{Listing}
\title{AAAI Press Formatting Instructions \\for Authors Using \LaTeX{} --- A Guide}
\author {
    Junli Jiang,\textsuperscript{\rm 1}
    Pavel Naumov\textsuperscript{\rm 2}
}
\affiliations {
    \textsuperscript{\rm 1} Institute of Logic and Intelligence, Southwest University, China\\
    \textsuperscript{\rm 2} University of Southampton, United Kingdom\\
    walk08@swu.edu.cn,  
    p.naumov@soton.ac.uk
}

\usepackage{amsmath,amssymb}

\newtheorem{definition}{Definition}
\newtheorem{lemma}{Lemma}
\newtheorem{theorem}{Theorem}

\newtheorem{claim}{Claim}
\newcommand{\K}{{\sf K}}

\newcommand{\B}{{\sf B}}

\renewcommand{\phi}{\varphi}
\usepackage{booktabs}

\newenvironment{proof}{\begin{trivlist}\item\noindent{\sc Proof.}}{\qed\end{trivlist}}
\newcommand{\qed}{\hfill $\boxtimes\hspace{2mm}$}

\newenvironment{proof-of-claim}{\begin{trivlist}\item\noindent{\sc Proof of Claim.}}{\hfill $\boxtimes\hspace{2mm}$\end{trivlist}}

\usepackage{algpseudocode}
\algnewcommand\algorithmicswitch{\textbf{switch}}
\algnewcommand\algorithmiccase{\textbf{case}}
\algnewcommand\algorithmicassert{\texttt{assert}}
\algnewcommand\Assert[1]{\State \algorithmicassert(#1)}%
\algdef{SE}[SWITCH]{Switch}{EndSwitch}[1]{\algorithmicswitch\ #1\ \algorithmicdo}{\algorithmicend\ \algorithmicswitch}%
\algdef{SE}[CASE]{Case}{EndCase}[1]{\algorithmiccase\ #1}{\algorithmicend\ \algorithmiccase}%
\algtext*{EndSwitch}%
\algtext*{EndCase}%

\title{We Think We Can!: Doxastic Strategies}

\title{Doxastic Strategies}

\title{Doxastic Strategies and the Triumph of
American Technology}

\title{The Triumph of
American Technology: Doxastic Strategies}

\title{Doxastic Coalition Power}

\title{Doxastic Coalitional Strategies}

\title{The Logic of Doxastic Strategies}

\begin{document}

\maketitle

\begin{abstract}
In many real-world situations, there is often not enough information to know that a certain strategy will succeed in achieving the goal, but there is a good reason to believe that it will. The paper introduces the term ``doxastic'' for such strategies.

The main technical contribution is a sound and complete logical system that describes the interplay between doxastic strategy and belief modalities. 
\end{abstract}

\section{Introduction}

During the 1991 Gulf War, American troops successfully used surface-to-air Patriot missiles to intercept Iraqi-modified Scud tactical ballistic missiles originally designed by the Soviet Union in the 1950s. In fact, on 15~February 1991, President George H. W. Bush travelled to the Patriot manufacturing plant in Andover, Massachusetts to praise the missiles' capability:
\begin{quote}
    You see, what has taken place here is a triumph of American technology 
    \dots we are witnessing a revolution in modern warfare, a revolution that will shape the way that we defend ourselves for decades to come. For years, we've heard that antimissile defenses won't work, that shooting down a ballistic missile is impossible -- like trying to ``hit a bullet with a bullet.'' Some people called it impossible; you called it your job. They were wrong, and you were right \dots Patriot is 41 for 42: 42 Scuds engaged, 41 intercepted.~\cite{b91url}
\end{quote}
Just 10 days after this speech, an American radar detected an Iraqi Scud approaching  Dhahran, Saudi Arabia.

Traditionally, knowledge is modelled through an indistinguishability relation that captures an agent's own observations. However, modern autonomous agents often obtain their knowledge from the data that they receive from other systems. This was the case on 25 February 1991. By analysing the position, $\Vec{x}$, and the velocity, $\Vec{v}$, of the Scud, the Patriot concluded that intercepting the Scud was within its capabilities. We write this as
\begin{equation}\label{8-aug-a}
    \K_{\Vec{x},\Vec{v}}(\text{``The Patriot has a strategy to destroy the Scud''})
\end{equation}
and say that dataset $\{\Vec{x},\Vec{v}\}$ informs the knowledge of the statement ``The Patriot has a strategy to destroy the Scud''. In general, we write $\K_X\phi$ if knowing the values of all variables in set $X$ is enough to conclude that statement $\phi$ is true. We refer to sets of data variables, like $X$, as {\em datasets}. Modality $\K_X$ for a set $X$ of {\em Boolean} variables is introduced by~\citet{gls15jair}. For arbitrary datasets, it is proposed by~\citet{bv21jpl}. The name ``data-informed knowledge'' is suggested by us in~\cite{jn22ai}. We also proposed a dynamic logic for this modality~\cite{djnz24jpl}.

Of course, for the Patriot to know that it has a strategy is not the same as to know what the strategy is. The distinction between ``knowing that a strategy exists'' and ``knowing what the strategy is'', in the case of the traditional (indistinguishability-relation-based) knowledge, has been studied in various logics of know-how strategies~\cite{aa16jlc,nt17aamas,fhlw17ijcai,nt18ai,nt18aaai,nt18aamas,fvw21rsl,lw21ai}. 

In our data-informed setting, in order for the Patriot to know the strategy (direction in which to launch a surface-to-air missile), it also must know the exact time, $t$, since the position, $\Vec{x}$, and the velocity, $\Vec{v}$, of the Scud has been measured. We write this as
\begin{equation}\label{8-aug-b}
   [\text{Patriot}]_{\Vec{x},\Vec{v},t}(\text{``Scud is destroyed''}) 
\end{equation}
and say that the dataset $\{\Vec{x},\Vec{v},t\}$ informs the strategy of Patriot to achieve the condition ``Scud is destroyed''.  In general, we write $[C]_X\phi$ if knowing the values of all variables in dataset $X$ is enough to know a strategy for coalition $C$ to achieve condition $\phi$. We introduced this modality in~\cite{jn22ai}. 

The values in the dataset $\{\Vec{x},\Vec{v},t\}$, obtained from the radar, gave the Patriot enough information to compute the strategy (direction of missile launch) that was guaranteed to destroy the Scud. Upon computing the strategy, the Patriot sent it to the launcher\dots Seconds later, Scud hit the barracks of the 14th Quartermaster Detachment of the US Army's 99th Infantry Division. 

To understand what went wrong that day in the skies above Dhahran, Saudi Arabia, one needs to acknowledge the fact that data does not always correctly reflect the state-of-affairs of the real world. Even if the Patriot knows the ``true'' values of $\Vec{x}$ and $\vec{v}$ as measured by the radar, these values might not reflect the correct position and velocity of Scud due to, for example, failure in the radar operation. In the case of our statement~(\ref{8-aug-a}), it is perhaps more sensible to say that by analysing the radar-provided values of $\Vec{x}$ and $\Vec{v}$ and {\em trusting} those values, the Patriot formed a {\em belief} that intercepting the Scud is within its capabilities:
\begin{equation}\label{8-aug-c}
    \B^{\Vec{x},\Vec{v}}_{\Vec{x},\Vec{v}}(\text{``The Patriot has a strategy to destroy the Scud''}).
\end{equation}
In general, we write $\B^T_X\phi$ if under the assumption of the trust in dataset $T$, the values of variables in the dataset $X$ inform belief $\phi$. We proposed this {\em trust-based belief} modality in~\cite{jn22ijcai-trust}. 

As it turns out, the statement $\B^\varnothing_X\phi$ is equivalent to the statement $\K_X\phi$. In other words, any belief, which is not based on trust, is knowledge. Any such belief, of course, is true. In general, just like the other types of beliefs, trust-based beliefs are not always true. They might be false if they are based on {\em trust} in {\em non-trustworthy} data. This, however, was {\em not} the case on 25 February 1991. The position and the velocity of the Scud communicated by the radar correctly reflected the physical characteristics of the approaching ballistic missile. Thus, the Patriot indeed had the capability to intercept the Scud.

As we have just observed, when the trust in data is taken into account, the knowledge statement~(\ref{8-aug-a}) becomes a belief statement~(\ref{8-aug-c}). Of course, a similar adjustment should be made to statement~(\ref{8-aug-b}). The Patriot did not actually {\em know} the strategy to destroy the Scud, it only had a {\em belief that a particular strategy will work}. This belief was based on the {\em trust} in the dataset $\{\Vec{x},\Vec{v},t\}$. We write this as:
\begin{equation}\label{9-aug-a}
[\text{Patriot}]^{\Vec{x},\Vec{v},t}_{\Vec{x},\Vec{v},t}(\text{``Scud is destroyed''}). 
\end{equation}
In general, we write $[C]^T_X\phi$ if, under the assumption of trust in dataset $T$, dataset $X$ informs a belief about some strategy of coalition $C$ that it will achieve condition $\phi$. We call such strategies ``doxastic'' (related to beliefs). Doxastic strategies are guaranteed to work if dataset $T$ is not only {\em trusted} but also {\em trustworthy}.

The Patriot system clock stored time (since boot) measured in tenths of a second. The actual time was computed by multiplying the stored value by $0.1$ using a 24-bit fixed point register. This introduced an error of about 0.000000095 seconds. By the time the Scud appeared in Dhahran, the system had been running for over 100 hours, creating an accumulated time difference between the radar's and the launcher's clocks of about $0.34$ seconds. A Scud travels over 500 meters in this time~\cite{m10url}. The strategy that the Patriot {\em believed} would destroy the Scud was based on {\em trusted} but {\em non-trustworthy} time data. The ``triumph of
American technology'' pointed the missile launcher in a completely wrong direction, killing 28 American soldiers~\cite{m10url} and wounding close to 100~\cite{a91nyt}. 
A year later, on 11 November 1992, in Greensburg, Pennsylvania, a memorial was dedicated to the soldiers killed by the Scud attack. The monument is facing in the direction of Dhahran, Saudi Arabia~\cite{tal17url}.

In this paper, we introduce formal semantics for doxastic strategies and a sound and complete logical system that describes the interplay between modalities $\B^T_X$ and $[C]^T_X$.

The paper is structured as follows. In the next section, we define the class of games that we use to model multiparty interactions. Next, we define the syntax and the formal semantics of our logical system. In Section~\ref{axioms section}, we list its axioms and inference rules. Section~\ref{big Completeness section} contains the proof of the completeness.  Section~\ref{ex ante ex post} discusses {\em ex ante} and {\em ex post} trust. Section~\ref{conclusion section} concludes. The proof of soundness and one of the lemmas from Section~\ref{big Completeness section} can be found in the appendix. 

\section{Games}

In this section, we introduce the class of games that is used later to give a formal semantics of our logical system. Throughout the paper, we assume a fixed set of atomic propositions and a fixed set of data variables $V$. In addition, we assume a fixed set of {\em actors} $\mathcal{A}$. We use the term ``actor'' instead of the more traditional ``agent'' to emphasise the fact that in our setting knowledge and beliefs come from data and they are not related to actors, who are only endowed with an ability to act.

By a {\em dataset} we mean any subset of $V$. Informally, given a game, we think about each data variable as having a value in each state of the game. However, formally, it is only important to know, for each two given states, if a variable has the same or different values in those states. Thus, we only need to associate an indistinguishability relation $\sim_x$ with each data variable $x\in V$. Informally, $w\sim_x u$ if data variable $x$ has the same value in states $w$ and $u$.  

To model the trustworthiness of data variables, for each state $w$ of a game we specify a dataset $\mathcal{T}_w\subseteq V$ consisting of the variables that are trustworthy (but not necessarily trusted) in the state $w$. We introduced this way to model the trustworthiness of data in~\cite{jn22ijcai-trust}. 

Finally, in the definition below and throughout the paper, by $Y^X$ we mean the set of all functions from set $X$ to set $Y$.

\begin{definition}\label{epistemic model}
A tuple $(W,\{\sim_x\}_{x\in V},\{\mathcal{T}_w\}_{w\in W},\Delta,M,\pi)$ is called a game if
\begin{enumerate}
    \item $W$ is a (possibly empty) set of states,
    \item $\sim_x$ is an indistinguishability equivalence relation on set $W$ for each data variable $x\in V$,
    \item $\mathcal{T}_w$ the set of data variables that are ``trustworthy'' in state $w\in W$,
    \item $\Delta$ is a nonempty set of ``actions'',
    \item $M\subseteq W\times \Delta^\mathcal{A}\times W$ is a ``mechanism'' of the game,
    \item $\pi(p)\subseteq W$ for each atomic proposition $p\in P$. 
\end{enumerate}
\end{definition}
By a {\em complete action profile} we mean an arbitrary element of the set $\Delta^\mathcal{A}$. By a {\em coalition} we mean an arbitrary subset $C\subseteq\mathcal{A}$ of actors. By an {\em action profile} of a coalition $C$ we mean an arbitrary element of the set $\Delta^C$.

The mechanism of the game $M$ specifies possible transitions of the game from one state to another. If $(w,\delta,u)\in M$, then under complete action profile $\delta$ from state $w$ the game can transition to state $u$. Note that we do not require the mechanism to be deterministic. We also do {\em not} require that for each state $w\in W$ and each complete action profile $\delta$ there is at least one state $u$ such that $(w,\delta,u)\in M$. If such state $u$ does not exist, then we interpret this as a termination of the game upon the execution of profile $\delta$ in state $w$.

\section{Syntax and Semantics}\label{syntax and semantics section}

Language $\Phi$ of our logical system is defined by the grammar
$$
\phi::= p\;|\;\neg\phi\;|\;(\phi\to\phi)\;|\;\B^T_X\phi\;|\;[C]^T_X\phi,
$$
where $p$ is an atomic proposition, $C\subseteq \mathcal{A}$ is a coalition, and $T,X\subseteq V$ are datasets. We read expression $\B^T_X\phi$ as ``under the assumption of trust in dataset $T$, dataset $X$ informs belief $\phi$''. And we read expression $[C]^T_X\phi$ as ``under the assumption of trust in dataset $T$, dataset $X$ informs a doxastic strategy of coalition $C$ to achieve $\phi$''.

We write $f=_C g$ if $f(x)=g(x)$ for each element $x\in C$ and $w\sim_X u$ if $w\sim_x u$ for each data variable $x\in X$.
\begin{definition}\label{sat}
For any formula $\phi\in\Phi$, and any state $w\in W$ of a game $(W,\{\sim_x\}_{x\in V},\{\mathcal{T}_w\}_{w\in W},\Delta,M,\pi)$, the satisfaction relation $w\Vdash\phi$ is defined as follows
\begin{enumerate}
    \item $w\Vdash p$, if $w\in \pi(p)$, 
    \item $w\Vdash \neg\phi$, if $w\nVdash \phi$,
    \item $w\Vdash \phi\to\psi$, if $w\nVdash \phi$ or $w\Vdash \psi$,
    \item $w\Vdash \B^T_{X}\phi$, if $u\Vdash \phi$ for each state $u\in W$ such that $w\sim_{X} u$ and $T\subseteq \mathcal{T}_{u}$,
    \item $w\Vdash[C]^T_X\phi$, when there is an action profile $s\in \Delta^C$ of coalition $C$ such that for all states $u,v\in W$ and each complete action profile $\delta\in\Delta^\mathcal{A}$, if $w\sim_{X}u$, $T\subseteq \mathcal{T}_u$, $s=_C \delta$, $(u,\delta,v)\in M$, and $T\subseteq \mathcal{T}_v$, then $v\Vdash\phi$.
\end{enumerate}
\end{definition}
Note that, by item~4 of the above definition, $w\Vdash \B^\varnothing_{X}\phi$ states that condition $\phi$ is satisfied in each state $u\in W$ indistinguishible from state $w$ by dataset $X$. Thus, $B^\varnothing_{X}$ is equivalent to data-informed knowledge modality $\K_X$ discussed in the introduction.

Item~5 above captures our informal intuition that $[C]^T_X\phi$ means that dataset $X$ informs the knowledge of a strategy that guarantees the achievement of $\phi$ in states where dataset $T$ is trustworthy. Note that there are multiple ways to formalise this: we can require  $T$ to be trustworthy in state $u$ (before the transition), in state $v$ (after the transition), or in both of these states. By including conditions $T\subseteq \mathcal{T}_u$ and $T\subseteq \mathcal{T}_v$ in item~5, we require $T$ to be trustworthy {\em ex ante} (before transition) and {\em ex post} (after transition). We discuss an alternative approach in Section~\ref{ex ante ex post}.

\begin{figure}[ht]
\begin{center}
\scalebox{.5}{\includegraphics{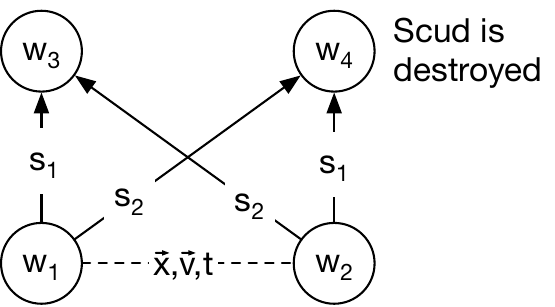}}
\caption{A game. Data variables $\Vec{x}$ and $\Vec{v}$ are trustworthy in all states. Data variable $t$ is trustworthy in states $w_2$, $w_3$, $w_4$. }\label{patriot figure}
\vspace{-3mm}
\end{center}
\end{figure}

Figure~\ref{patriot figure} depicts a (very simplistic) game capturing our introductory game. This game has a single actor, the Patriot. The actual state is $w_1$, in which data variable $t$ is not trustworthy. The only state which is $\{\Vec{x},\Vec{v},t\}$-indistinguishable from the current state and in which all variables in dataset $\{\Vec{x},\Vec{v},t\}$ are trustworthy is state $w_2$. Note that, in state $w_2$, strategy $s_1$ can be used to destroy the Scud. Thus, in the current state $w_1$, dataset $\{\Vec{x},\Vec{v},t\}$ informs the belief that the Patriot can use strategy $s_1$ to destroy the Scud. Hence, in this game, formula~(\ref{9-aug-a}) is satisfied in state $w_1$.

\section{Axioms}\label{axioms section}

In addition to propositional tautologies in language $\Phi$, our logical system contains the following axioms.

\begin{enumerate}
    \item Truth: $\B^\varnothing_X\phi\to \phi$,
    \item Negative Introspection: $\neg\B^T_X\phi\to \B^\varnothing_X\neg\B^T_X\phi$,
    \item Distributivity: $\B^T_X(\phi\to\psi)\to(\B^T_X\phi\to\B^T_X\psi)$,
    \item Trust: $\B^T_X(\B^T_Y\phi\to\phi)$,
    \item Monotonicity: $\B^T_X\phi\to \B^{T'}_{X'}\phi$ and $[C]^T_X\phi\to [C']^{T'}_{X'}\phi$, where $T\subseteq T'$, $X\subseteq X'$, and $C\subseteq C'$,
    \item Cooperation: $[C]^T_X(\phi\!\to\!\psi)\to([D]^T_X\phi\!\to\![C\cup D]^T_X\psi)$, where $C\cap D=\varnothing$,
    \item Strategic Introspection: $[C]^T_X\phi\leftrightarrow\B^T_X[C]^T_X\phi$,
    \item Belief in Unavoidability: $\B^T_X[\varnothing]^T_Y\phi\to [\varnothing]^T_X\phi$,
    \item Public Belief: $\B^T_\varnothing\phi\to[\varnothing]^T_\varnothing\phi$.
\end{enumerate}

The Truth axiom, the Negative Introspection axiom, the Distributivity axiom, the Trust axiom and the belief part of the Monotonicity axiom are the axioms for the trust-based beliefs as stated in~\cite{jn22ijcai-trust}. The most non-trivial among them is the Trust axiom. To understand the meaning of this axiom, note that although the principle $\B^T_X\phi\to\phi$ is not universally true, it is true in all states where dataset $T$ is trustworthy. This observation is captured in the Trust axiom.

The Cooperation axiom is a variation of the axiom from the Logic of Coalition Power~\cite{p02}. The Strategic Introspection axiom states that dataset $X$ informs a doxastic strategy iff $X$ informs a belief in having a doxastic strategy. 
The Strategic Introspection axiom and the Belief in Unavoidability axiom (which we discuss below) can both be stated in alternative forms where the superscript of modality $\B$ is the empty set $\varnothing$.

To understand the meaning of the Belief in Unavoidability axiom, note that statement $[\varnothing]^\varnothing_X\phi$ means that dataset $X$ informs the knowledge that condition $\phi$ is {\em unavoidably true} in the next state. Similarly, $[\varnothing]^T_X\phi$ means that, under the assumption of trust in dataset $T$, dataset $X$ informs the {\em belief} that condition $\phi$ is unavoidably true in the next state. The Belief in Unavoidability axiom states that {\em if $X$ informs the belief that $Y$ informs the belief in unavoidability of $\phi$, then $X$ itself informs this belief in unavoidability of $\phi$}. 

Finally, observe that $\B^\varnothing_\varnothing\phi$ means that statement $\phi$ is true in all states of the game. In other words, $\phi$ is public knowledge in the game. The Public Belief axiom states that if condition $\phi$ holds in all words where $T$ is trustworthy, then it is believed to be unavoidably true in the next state.

We write $\vdash\phi$ and say that formula $\phi\in\Phi$ is a theorem of our system if it is derivable from the above axioms using the Modus Ponens and the  Necessitation inference rules:

$$
\dfrac{\phi,\phi\to\psi}{\psi}
\hspace{20mm}
\dfrac{\phi}{\B^\varnothing_\varnothing\phi}.
$$
In addition to the unary relation $\vdash \phi$, we also consider a binary relation $F\vdash\phi$. We write $F\vdash\phi$ if formula $\phi\in\Phi$ is provable from the set of formulae $F\subseteq\Phi$ and the theorems of our logical system using only the Modus Ponens inference rule. It is easy to see that statement $\varnothing\vdash\phi$ is equivalent to $\vdash\phi$.
We say that set $F$ is inconsistent if there is a formula $\phi\in F$ such that $F\vdash\phi$ and $F\vdash\neg\phi$.

The proof of the following theorem is in the appendix.

\begin{theorem}[soundness]\label{strong soundness theorem}
If $\vdash\phi$, then $w\Vdash\phi$ for each state $w$ of an arbitrary game.
\end{theorem}

We conclude this section with two auxiliary results that will be used later in the proof of the completeness. The first of them is a form of the Positive Introspection principle for our belief modality $\B$.
\begin{lemma}\label{positive introspection lemma}
$\vdash \B^T_X\phi\to\B^\varnothing_X\B^T_X\phi$. 
\end{lemma}
\begin{proof}
Formula $\B^\varnothing_X\neg\B^T_X\phi\to\neg\B^T_X\phi$ is an instance of the Truth axiom. Thus, $\vdash \B^T_X\phi\to\neg\B^\varnothing_X\neg\B^T_X\phi$, by contraposition. Hence, taking into account the following instance $\neg\B^\varnothing_X\neg\B^T_X\phi\to\B^\varnothing_X\neg\B^\varnothing_X\neg\B^T_X\phi$ of  the Negative Introspection axiom,
we have 
\begin{equation}\label{pos intro eq 2}
\vdash \B^T_X\phi\to\B^\varnothing_X\neg\B^\varnothing_X\neg\B^T_X\phi.
\end{equation}
Also, the formula $\neg\B^T_X\phi\to\B^\varnothing_X\neg\B^T_X\phi$ is an instance of the Negative Introspection axiom. Thus, by contraposition, $\vdash \neg\B^\varnothing_X\neg\B^T_X\phi\to \B^T_X\phi$. Hence, by the Necessitation inference rule,
$\vdash \B^\varnothing_\varnothing(\neg\B^\varnothing_X\neg\B^T_X\phi\to \B^T_X\phi)$.
Thus, 
$\vdash \B^\varnothing_X(\neg\B^\varnothing_X\neg\B^T_X\phi\to \B^T_X\phi)$
by the Monotonicity axiom and the Modus Ponens inference rule. Thus, by  the Distributivity axiom and the Modus Ponens inference rule it follows that
$
  \vdash \B^\varnothing_X\neg\B^\varnothing_X\neg\B^T_X\phi\to \B^\varnothing_X\B^T_X\phi
$.
 The latter, together with statement~(\ref{pos intro eq 2}), implies the statement of the lemma by propositional reasoning.
\end{proof}

The proof of the following lemma is in the appendix.

\begin{lemma}\label{super distributivity}
If $\phi_1,..,\phi_n\vdash\psi$, then $\B^T_X\phi_1,..,\B^T_X\phi_n\vdash\B^T_X\psi$.
\end{lemma}

\section{Strong Completeness}\label{big Completeness section}

The proof of the completeness theorem is split into four subsections. In Subsection~\ref{canonical model section}, we define the canonical game. The truth lemma for the canonical game and the final step of the proof are given in Subsection~\ref{Final Steps section}. In the two subsections before, we state and prove auxiliary lemmas used in the proof of the truth lemma. A highly non-trivial proof of one of these auxiliary lemmas is located in the appendix.

\subsection{Canonical Game}\label{canonical model section}

Throughout the rest of the paper, we assume that $\mathcal{B}$ is any set of cardinality larger than $\mathcal{A}$, such as, for example, the powerset $\mathcal{P}(\mathcal{A})$. We explain the need for $\mathcal{B}$ after Definition~\ref{canonical W} below.  

In this section, towards the proof of completeness, for an arbitrary maximal consistent set of formulae $F_0\subseteq\Phi$ and an arbitrary dataset $T_0\subseteq V$, we define a ``canonical'' game $G(T_0,F_0)=(W,\{\sim_x\}_{x\in V},\{\mathcal{T}_w\}_{w\in W},\Delta,M,\pi)$.

The set of states in the canonical game is defined using the tree construction which goes back to the proof of the completeness theorem for epistemic logic of distributed knowledge~\cite{fhv92jacm}. Informally, in the tree construction, the states are the nodes of a tree. Formally, the states are defined first as finite sequences; the tree structure on these sequences is specified later. 

\begin{definition}\label{canonical W}
$W$ is the set of all sequences $T_0,F_0,X_1$, $b_1,T_1,F_1$, $\dots,X_n,b_n,T_n,F_n$ such that $n\ge 0$ and 
\begin{enumerate}
    \item $F_i\subseteq \Phi$ is a maximal consistent set of formulae for each integer $i\ge 1$,
    \item $T_i,X_i\subseteq V$ are datasets for each integer $i\ge 1$,
    \item $b_i\in\mathcal{B}$ for each integer $i\ge 1$,
    \item $\{\phi\;|\; \B^\varnothing_{X_i}\phi\in F_{i-1}\}\subseteq F_i$ for each integer $i\ge 1$,
    \item $\B^{T_{i}}_Y\psi\to\psi\in F_i$ for each dataset $Y\subseteq V$, each formula $\psi\in\Phi$, and each integer $i\ge 0$.
\end{enumerate}
\end{definition}
For any two states $u,w\in W$ of the form
\begin{eqnarray*}
    u&\!\!=\!\!&T_0,F_0,\dots,X_{n-1},b_{n-1},T_{n-1},F_{n-1}\\
    w&\!\!=\!\!&T_0,F_0,\dots,X_{n-1},b_{n-1},T_{n-1},F_{n-1},X_n,b_n,T_n,F_n,
\end{eqnarray*}
we say that the states are {\em adjacent}. Note that the adjacency relation forms a tree structure (undirected graph without cycles) on set $W$. We say that the edge $(u,w)$ is {\em labelled} with each data variable in set $X_n$. By $F(w)$ and $T(w)$ we denote sets $F_n$ and $T_n$, respectively.
Informally, it is convenient to visualise this tree with the edge marked by the pair $X_n,b_n$ and node $w$ marked by the pair $T_n,F_n$. Figure~\ref{Canonical-Tree figure} depicts a fragment of such visual representation. In this figure, the node $T_0,F_0,X_1,b_1,T_1,F_1$ is adjacent to the node $T_0,F_0,X_1,b_1,T_1,F_1,X_3,b_3,T_3,F_3$. The edge between them is {\em labelled} with each {\em variable} in set $X_3$.

By {\em clones} we call any two nodes
\begin{eqnarray*}
    w_1&\!\!\!=\!\!\!&T_0,F_0,\dots,X_{n-1},b_{n-1},T_{n-1},F_{n-1},X_n,b_n,T_n,F_n\\
    w_2&\!\!\!=\!\!\!&T_0,F_0,\dots,X_{n-1},b_{n-1},T_{n-1},F_{n-1},X_n,b'_n,T_n,F_n
\end{eqnarray*}
that differ only by the {\em last} $b$-value. For example, in Figure~\ref{Canonical-Tree figure}, nodes $T_0,F_0,X_1,b_1,T_1,F_1,X_3,b_3,T_3,F_3$
and
$T_0,F_0,X_1$, $b_1,T_1,F_1,X_3,b_6,T_3,F_3$ are clones. The purpose of set $\mathcal{B}$ in our construction is to guarantee that each node (except for the root) has more clones than the cardinality of set $\mathcal{A}$.

\begin{figure}
\begin{center}
\scalebox{.5}{\includegraphics{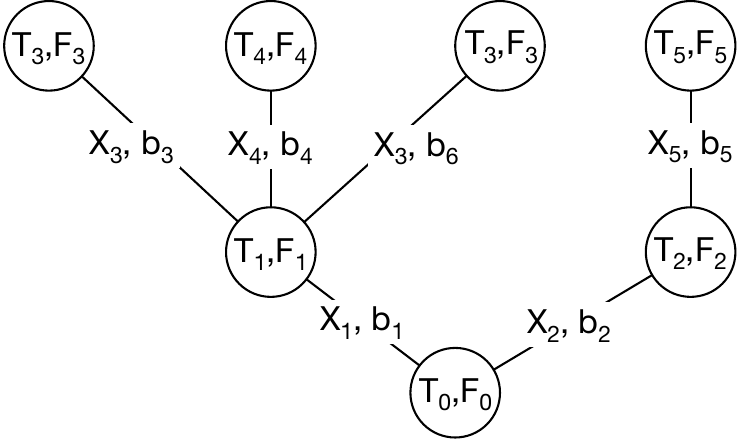}}
\caption{Fragment of a canonical tree.}\label{Canonical-Tree figure}
\vspace{-5mm}
\end{center}
\end{figure}

Recall that a {\em simple} path in a graph is a path without self-intersections and that in any tree there is a unique simple path between any two nodes.

\begin{definition}\label{canonical sim}
For any states $w,w'\in W$ and any variable $x\in V$, let $w\sim_x w'$ if every edge along the unique simple path between nodes $w$ and $w'$ is labelled with variable~$x$.
\end{definition}
\begin{lemma}
Relation $\sim_x$ is an equivalence relation on set $W$ for each data variable $x\in V$.
\end{lemma}

As we prove later in Lemma~\ref{truth lemma}, for any state $w\in W$, set $F(w)$ is the set of all formulae that are satisfied in state $w$ of our canonical model. At the same time, set $T(w)$ is the set of all data variables that are trustworthy in state $w$: 

\begin{definition}\label{canonical T}
    $\mathcal{T}_w=T(w)$ for any state $w\in W$.
\end{definition}

The above definition explains the intuition behind item~5 of Definition~\ref{canonical W}. Indeed, the item states $\B^{T(w)}_Y\phi\to\phi\in F(w)$. Intuitively, this means that if $\B^{T(w)}_Y\phi$ is true in a state $w$, where set $T(w)$ is trustworthy by Definition~\ref{canonical T}, then statement $\phi$ must be true in state $w$.

In this paper, as in several other works that extend Marc Pauly's logic of coalitional power~\cite{gd06tcs,nt17aamas,ge18aamas}, the mechanism of the canonical game is using the voting construction. In the standard version of this construction, if the goal of the coalition is to achieve $\phi$, then all members of the coalition use action $\phi$. Informally, this could be interpreted as ``voting for $\phi$''. In our case, a strategy is informed by a dataset $X$. Thus, for the strategy to work the coalition should somehow prove that it has access to the values of dataset $X$ in the current state. To achieve this, we require each vote to be ``signed'' by naming a state that belongs to the same $X$-equivalence class as the current state. Thus, each action consists of a formula and a state (``signature''). Similar constructions are used in~\cite{nt18ai,jn22ai}.

\begin{definition}
$\Delta=\Phi\times W$.
\end{definition}

If $d$ is a pair $(x,y)$, then by $pr_1(d)$ and $pr_2(d)$ we mean elements $x$ and $y$,  respectively.

  \begin{definition}\label{canonical mechanism}
Mechanism $M$ consists of all triples $(w,\delta,v)\in W\times\Delta^\mathcal{A}\times W$ such that for each formula $[C]^T_X\phi\in F(w)$, if 
\begin{enumerate}
    \item $pr_1(\delta(a))=\phi$ for each actor $a\in C$ and
    \item $w\sim_X pr_2(\delta(a))$ for each actor $a\in C$,
    \item $T\subseteq \mathcal{T}_w$,
    \item $T\subseteq \mathcal{T}_v$,
\end{enumerate}
then $\phi\in F(v)$.
\end{definition}

\begin{definition}\label{canonical pi}
$\pi(p)=\{w\in W\;|\; p\in F(w)\}$ for each atomic proposition $p\in P$.
\end{definition}

This concludes the definition of the canonical game. 
As usual, the key step in the proof of a completeness is a ``truth lemma'' proven by induction on structural complexity of a formula. In our case, this is Lemma~\ref{truth lemma}. To improve readability of the proof of Lemma~\ref{truth lemma}, we separate the non-trivial induction steps into separate lemmas stated in the next two subsections.

\subsection{Properties of the belief modality}

In this subsection, we state and prove lemmas for modality~$\B$ used in the induction step of the proof of Lemma~\ref{truth lemma}.

\begin{lemma}\label{single transfer lemma}
For any formula $\B^T_Y\phi\in\Phi$ and any states
\begin{eqnarray*}
w'&=&T_0,F_0,\dots,b_{n-1},T_{n-1},F_{n-1},\\
w&=&T_0,F_0,\dots,b_{n-1},T_{n-1},F_{n-1},X_n,b_n,T_n,F_n
\end{eqnarray*}
if $Y\subseteq X_n$, then $\B^T_Y\phi\in F(w')$ \,iff\, $\B^T_Y\phi\in F(w)$.
\end{lemma}
\begin{proof} $(\Rightarrow):$ Let
$\B^T_Y\phi\in F(w')$. Then,
$\B^T_Y\phi\in F_{n-1}$. Thus, by Lemma~\ref{positive introspection lemma} and the Modus Ponens inference rule, 
$F_{n-1}\vdash \B^\varnothing_Y\B^T_Y\phi$. 
Hence, 
$F_{n-1}\vdash \B^\varnothing_{X_n}\B^T_Y\phi$
by the assumption $Y\subseteq X_n$ of the lemma, the Monotonicity axiom, and the Modus Ponens inference rule.
Thus, $\B^\varnothing_{X_n}\B^T_Y\phi\in F_{n-1}$ because $F_{n-1}$ is a maximal consistent set. Then, $\B^T_Y\phi\in F_{n}$ by item~4 of Definition~\ref{canonical W}. Therefore, $\B^T_Y\phi\in F(w)$. 

\vspace{1mm}\noindent$(\Leftarrow):$ 
Let
$\B^T_Y\phi\notin F(w')$.
Thus, $\B^T_Y\phi\notin F_{n-1}$. Hence, $\neg\B^T_Y\phi\in F_{n-1}$ because $F_{n-1}$ is a maximal consistent set of formulae. Then, $F_{n-1}\vdash \B^\varnothing_Y\neg\B^T_Y\phi$ by the Negative Introspection axiom and the Modus Ponens inference rule. 
Thus, 
$F_{n-1}\vdash \B^\varnothing_{X_n}\neg\B^T_Y\phi$ by the assumption $Y\subseteq X_n$ of the lemma, the Monotonicity axiom, and the Modus Ponens inference rule.
Hence, because set $F_{n-1}$ is maximal, $\B^\varnothing_{X_n}\neg\B^T_Y\phi\in F_{n-1}$. Then, $\neg\B^T_{Y}\phi\in F_{n}$ by item~4 of Definition~\ref{canonical W}.
Thus,
$\B^T_Y\phi\notin F_{n}$, because set $F_n$ is consistent. Therefore, $\B^T_Y\phi\notin F(w)$.
\end{proof}

\begin{lemma}\label{child all}
For any states $w,u\in W$ and any formula $\B^T_X\phi\in F(w)$, if $\mbox{$w\sim_X u$}$ and $T\subseteq \mathcal{T}_u$, then $\phi\in F(u)$. 
\end{lemma}
\begin{proof}
By Definition~\ref{canonical sim}, the assumption $w\sim_X u$ implies that each edge along the unique path between nodes $w$ and $u$ is labelled with each variable in dataset $X$. Then, the assumption
$\B^T_X\phi\in F(w)$ implies 
$
    \B^T_X\phi\in F(u)
$
by applying Lemma~\ref{single transfer lemma} to each edge along this path. Note that the assumption $T\subseteq \mathcal{T}_u$ of the lemma implies that $T\subseteq T(u)$ by Definition~\ref{canonical T}. Hence, 
$
    F(u) \vdash \B^{T(u)}_X\phi 
$
by the Monotonicity axiom and the Modus Ponens inference rule. 
Thus,
$
F(u) \vdash \phi
$
by item~5 of Definition~\ref{canonical W} and the Modus Ponens inference rule. Therefore, $\phi\in F(u)$ because the set $F(u)$ is maximal.
\end{proof}

\begin{lemma}\label{child exists}
For any $w\in W$ and any formula $\B^T_X\phi\notin F(w)$, there exists a state $u\in W$ such that $w\sim_X u$, $T=\mathcal{T}_u$, and $\phi\notin F(u)$.
\end{lemma}
\begin{proof}
Consider the set of formulae
\begin{eqnarray}\label{G definition}
H&=&\{\neg\phi\}\cup \{\psi\;|\;\B^\varnothing_X\psi\in F(w)\}\nonumber\\
&&\cup \{\B^T_Y\chi\to \chi\;|\;Y\subseteq V, \chi\in\Phi\}.
\end{eqnarray}
\begin{claim}Set $H$ is consistent.
\end{claim}
\begin{proof-of-claim}
Assume the opposite. Hence, there are formulae $\chi_1,..,\chi_n\!\in\!\Phi$, datasets $Y_1,..,Y_n\!\subseteq\! V$, and formulae
\begin{equation}\label{19nov-a}
\B^\varnothing_X\psi_1,\dots,\B^\varnothing_X\psi_m\in F(w)
\end{equation}
such that
$$
\B^T_{Y_1}\chi_1\to \chi_1,\dots,\B^T_{Y_n}\chi_n\to \chi_n,
\psi_1,\dots,\psi_m\vdash \phi. 
$$
Thus, by Lemma~\ref{super distributivity},
\begin{eqnarray*}
&&\hspace{-10mm}\B^T_X(\B^T_{Y_1}\chi_1\to \chi_1),\dots,\B^T_X(\B^T_{Y_n}\chi_n\to \chi_n),\\
&&\hspace{+30mm}\B^T_X\psi_1,\dots,\B^T_X\psi_m\vdash \B^T_X\phi. 
\end{eqnarray*}
Hence, 
$
\B^T_X\psi_1,\dots,\B^T_X\psi_m\vdash \B^T_X\phi
$
by the Trust axiom applied $n$ times. 
Then,
$
\B^\varnothing_X\psi_1,\dots,\B^\varnothing_X\psi_m\vdash \B^T_X\phi
$
by the Monotonicity axiom and the Modus Ponens inference rule applied $m$ times.
Thus,
$
F(w)\vdash \B^T_X\phi
$
due to statement~(\ref{19nov-a}).
Hence, $\B^T_X\phi\in F(w)$ because the set $F(w)$ is maximal, which contradicts the assumption $\B^T_X\phi\notin F(w)$ of the lemma.
\end{proof-of-claim}

Let $H'$ be any maximal consistent extension of set $H$ and $b$ be any element of set $\mathcal{B}$.
Suppose that $w=T_0,F_0$, \dots, $X_n,b_n,T_n,F_n$.
Consider sequence
\begin{equation}\label{u choice}
    u=T_0,F_0,\dots,X_n,b_n,T_n,F_n,X,b,T,H'.
\end{equation}
Note that $u\in W$ by Definition~\ref{canonical W}, equation~(\ref{G definition}), and the choice of set $H'$ as an extension of set $H$. Also, observe that $w\sim_X u$ by Definition~\ref{canonical sim} and equation~(\ref{u choice}). In addition, $T=T(u)=\mathcal{T}_u$ by equation~(\ref{u choice}) and Definition~\ref{canonical T}. Finally, $\neg\phi\in H\subseteq H'=F(u)$ by equation~(\ref{G definition}), the choice of $H'$ as an extension of $H$, and equation~(\ref{u choice}). Therefore, $\phi\notin F(u)$ because the set $F(u)$ is consistent. This concludes the proof of the lemma.
\end{proof}

\subsection{Properties of the doxastic strategy modality}

We now state lemmas for modality~$[C]$ that will be used in the induction step of the proof of Lemma~\ref{truth lemma}.

\begin{lemma}\label{H child all}
For an arbitrary state $w\in W$ and any formula $[C]^T_X\phi\in F(w)$ there is an action profile $s\in \Delta^C$ such that for all states $w',v\in W$ and each complete action profile $\delta\in\Delta^\mathcal{A}$ if $w\sim_{X}w'$, $s=_C \delta$, $T\subseteq \mathcal{T}_{w'}$, $T\subseteq \mathcal{T}_{v}$, and $(w',\delta,v)\in M$, then $\phi\in F(v)$.
\end{lemma}
\begin{proof}
Let action profile $s\in\Delta^C$ be  such that
\begin{equation}\label{apr21-a}
    s(a)=(\phi,w)
\end{equation}
for each actor $a\in C$. Consider any states $w',v\in W$ and any complete action profile $\delta\in\Delta^\mathcal{A}$ such that 
\begin{eqnarray}
\hspace{-8mm}&&w\sim_{X}w',\label{apr21-b-part1}\\
\hspace{-8mm}&&s=_C \delta,\;T\subseteq \mathcal{T}_{w'},\;T\subseteq \mathcal{T}_{v},\; \mbox{and}\; (w',\delta,v)\in M. \label{apr21-b}
\end{eqnarray}
It suffices to show that $\phi\in F(v)$.

The assumption $[C]^T_X\phi\in F(w)$ of the lemma implies  $F(w)\vdash \B^T_X[C]^T_X\phi$ by the Strategic Introspection axiom and propositional reasoning. 
Hence, $\B^T_X[C]^T_X\phi\in F(w)$ because set $F(w)$ is maximal. Thus, $[C]^T_X\phi\in F(w')$ by Lemma~\ref{child all} and assumption~(\ref{apr21-b-part1}) and the part $T\subseteq \mathcal{T}_{w'}$ of assumption~(\ref{apr21-b}). Therefore, $\phi\in F(v)$ by Definition~\ref{canonical mechanism}, and assumptions~(\ref{apr21-a}) and~(\ref{apr21-b}).
\end{proof}

The proof of the next lemma can be found in the appendix.

\begin{lemma}\label{H child exists}
For an arbitrary state $w\in W$, any formula $\neg[C]^T_X\phi\in F(w)$, and any action profile $s\in\Delta^C$, there are states $w',v\in W$ and a complete action profile $\delta\in\Delta^\mathcal{A}$ such that $w\sim_{X} w'$, $s=_C\delta$, $T\subseteq \mathcal{T}_{w'}$, $T\subseteq \mathcal{T}_{v}$, $(w',\delta,v)\in M$, and $\phi\notin F(v)$. 
\end{lemma}

\subsection{Final Steps}\label{Final Steps section}

We are now ready to prove the truth lemma for our logical system and the strong completeness of the system.

\begin{lemma}\label{truth lemma}
$w\Vdash \phi$ iff $\phi\in F(w)$ for each state $w\in W$ and each formula $\phi\in \Phi$.
\end{lemma}
\begin{proof}
We prove the statement by induction on the complexity of formula~$\phi$. 
Suppose that formula $\phi$ is an atomic proposition $p$. Note that $w\Vdash p$ iff $w\in \pi(p)$ by item~1 of Definition~\ref{sat}. At the same time, $w\in \pi(p)$ iff $p\in F(w)$ by Definition~\ref{canonical pi}. Therefore, $w\Vdash p$ iff $p\in F(w)$.

If formula $\phi$ is a negation or an implication, then the statement of the lemma follows from the maximality and the consistency of the set $F(w)$, items 2 and 3 of Definition~\ref{sat}, and the induction hypothesis in the standard way.

\vspace{1mm}
Suppose that formula $\phi$ has the form $\B^T_X\psi$. 

\vspace{1mm}
\noindent$(\Rightarrow):$ Assume that $\B^T_X\psi\notin F(w)$. Then, $\neg\B^T_X\psi\in F(w)$ because $F(w)$ is a maximal consistent set of formulae. Thus, by Lemma~\ref{child exists}, there is a state $w'\in W$ such that $w\sim_{X} w'$, $T\subseteq \mathcal{T}_{w'}$, and $\neg\psi\in F(w')$. Hence, $\psi\notin F(w')$ because set $F(w')$ is consistent. Then, $w'\nVdash\psi$ by the induction hypothesis. Therefore, $w\nVdash\B^T_X\psi$ by item~4 of Definition~\ref{sat} and statements $w\sim_{X} w'$ and $T\subseteq \mathcal{T}_{w'}$. 

\vspace{1mm}
\noindent$(\Leftarrow):$ Assume that $\B^T_X\psi\in F(w)$. Consider any state $w'$ such that $w\sim_{X} w'$ and $T\subseteq \mathcal{T}_{w'}$. By item~4 of Definition~\ref{sat}, it suffices to show that $w'\Vdash \psi$, which is true by Lemma~\ref{child all}.

\vspace{1mm}
Finally, suppose that formula $\phi$ has the form $[C]^T_X\psi$. 

\vspace{1mm}
\noindent
$(\Rightarrow):$ Assume that $[C]^T_X\psi\notin F(w)$. Thus, $\neg [C]^T_X\psi\in F(w)$ because set $F(w)$ is maximal. Hence, by Lemma~\ref{H child exists}, for any action profile $s\in\Delta^C$, there are states $w',v\in W$ and a complete action profile $\delta\in\Delta^\mathcal{A}$ such that $w\sim_{X} w'$, $s=_C\delta$, $T\subseteq \mathcal{T}_{w'}$, $T\subseteq \mathcal{T}_{v}$, $(w',\delta,v)\in M$, and $\psi\notin F(v)$. Then, by the induction hypothesis, for any action profile $s\in\Delta^C$, there are states $w',v\in W$ and a complete action profile $\delta\in\Delta^\mathcal{A}$ such that $w\sim_{X} w'$, $s=_C\delta$, $T\subseteq \mathcal{T}_{w'}$, $T\subseteq \mathcal{T}_{v}$, $(w',\delta,v)\in M$, and $v\nVdash\psi$. Therefore, $w\nVdash[C]^T_X\psi$ by item~5 of Definition~\ref{sat}.

\vspace{1mm}
\noindent$(\Leftarrow):$ Assume $[C]^T_X\psi\in F(w)$. Thus, by Lemma~\ref{H child all}, there is an action profile $s\in \Delta^C$ such that for all states $w',v\in W$ and each complete action profile $\delta\in\Delta^\mathcal{A}$ if $w\sim_{X}w'$, $s=_C \delta$, $T\subseteq \mathcal{T}_{w'}$, $T\subseteq \mathcal{T}_{v}$, and $(w',\delta,v)\in M$, then $\psi\in F(v)$. Hence, by the induction hypothesis, there is an action profile $s\in \Delta^C$ such that for all states $w',v\in W$ and each complete action profile $\delta\in\Delta^\mathcal{A}$ if $w\sim_{X}w'$, $s=_C \delta$, $T\subseteq \mathcal{T}_{w'}$, $T\subseteq \mathcal{T}_{v}$, and $(w',\delta,v)\in M$, then $v\Vdash\psi$. Therefore, $w\Vdash[C]^T_X\psi$ by item~5 of Definition~\ref{sat}.
\end{proof}

\begin{theorem}[strong completeness]
For any set of formulae $F\subseteq \Phi$ and any formula $\phi\in\Phi$, if $F\nvdash \phi$, then there is a state $w$ of a game such that $w\Vdash f$ for each formula $f\in F$ and $w\nVdash \phi$.
\end{theorem}
\begin{proof}
The assumption $F\nvdash \phi$ implies that the set $F\cup\{\neg\phi\}$ is consistent. Let $F_0$ be any maximal consistent extension of this set. Consider the canonical game $G(\varnothing,F_0)$. 

First, we show that the sequence $\varnothing,F_0$ is a state of this canonical game. By Definition~\ref{canonical W}, it suffices to show that $\B^\varnothing_Y\psi\to\psi\in F_0$ for each dataset $Y\subseteq V$ and each formula $\psi\in\Phi$. The last statement is true by the Truth axiom and because set $F_0$ is maximal.

Finally, note that $\phi\notin F_0$ because set $F_0$ is consistent and $\neg\phi\in F_0$. Therefore, by Lemma~\ref{truth lemma} and because $F\subseteq F_0$, it follows that $\varnothing,F_0\Vdash f$ for each formula $f\in F$ and also $\varnothing,F_0\nVdash \phi$.
\end{proof}

\section{Ex Ante and Ex Post Trust}\label{ex ante ex post}

In item~5 of Definition~\ref{sat}, we require that $T\subseteq \mathcal{T}_u$ and $T\subseteq \mathcal{T}_v$. In other words, we apply the assumption of trustworthiness of dataset $T$ {\em ex ante} (before action) and {\em ex post} (after action). In general, one can consider that different datasets, $A$ and $P$, are required to be trustworthy ex ante and ex post, respectively. That leads to a more general modality $[C]^{A,P}_X\phi$, defined below:

\vspace{1mm}\noindent{\em
$w\Vdash[C]^{A,P}_X\phi$, when there is an action profile $s\in \Delta^C$ of coalition $C$ such that for all states $u,v\in W$ and each complete action profile $\delta\in\Delta^\mathcal{A}$ if $w\sim_{X}u$, $A\subseteq \mathcal{T}_u$, $s=_C \delta$, $(u,\delta,v)\in M$, and $P\subseteq \mathcal{T}_v$, then $v\Vdash\phi$.
}

One might wonder which of the data variables among $\Vec{x}$, $\Vec{v}$, and $t$ should be trusted ex ante and which ex post for statement~\eqref{9-aug-a} to be true in our introductory example. Unfortunately, we do not have access to Patriot code to answer this question, but we suspect that the Patriot missile constantly adjusts the trajectory based on the current speed and position of the target. To model such behaviour one would need to use multi-step games instead of one-shot (strategic) games that we consider in this paper. 

To illustrate ex ante and ex post trust, let us consider a different example where a governing body consisting of 25 members is about to vote on passing a certain regulation. Suppose that each of them votes {\em yes} or {\em no} by a paper ballot. After the vote, the ballots are counted by a tallyman and the number of {\em yes} votes, denoted by $n$, is announced. If $n$ is more than 12, then the regulation is approved, see Figure~\ref{vote figure}. 

\begin{figure}[ht]
\begin{center}
\scalebox{.5}{\includegraphics{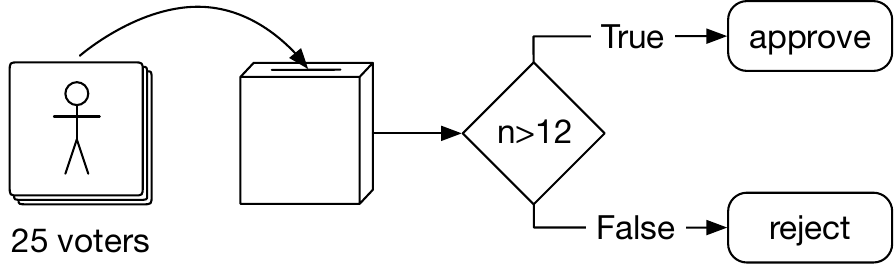}}
\caption{Voting game.}\label{vote figure}
\end{center}
\end{figure}

Let us now further assume that a newspaper contacts all 25 members and ask them how they plan to vote. Suppose that 10 members state that they plan to support the proposal, 10 members say that they plan to vote against the proposal, and 5 members are undecided. 
Alice, Bob, and Cathy are among those 5 who are undecided. It appears that the coalition consisting of the three of them has a strategy to sway the outcome of the vote either way. If all three of them vote {\em yes}, then the regulation is approved. If they vote {\em no}, then it is rejected. Note, however, that both of these strategies are {\em doxastic}: they rely on the {\em ex ant} trust in the newspaper data and  on the {\em ex post} trust in the tallyman computing $n$ correctly. Let data variables $yea$ and $nay$ represent the number of members, according to the newspaper, who plan to vote for and against, respectively. Note that, in the current world $w$, $yea=nay=10$. Thus, 
\begin{align}
w\Vdash& [\text{Alice,Bob,Cathy}]^{\{yea\},\{n\}}_{\{yea\}}\,\text{``Approved''},\label{27-june-a}\\
w\Vdash& [\text{Alice,Bob,Cathy}]^{\{nay\},\{n\}}_{\{nay\}}\,\text{``Rejected''}\nonumber.
\end{align}
All subscripts and superscripts are important in both formulae above. For example, if subscript $yea$ is dropped in statement~\eqref{27-june-a}, then the statement is no longer true:
$$
w\nVdash [\text{Alice,Bob,Cathy}]^{\{yea\},\{n\}}_{\varnothing}\,\text{``Approved''}.
$$
This is because, although the doxastic strategy exists, the empty dataset {\em does not inform} its existence.

If $yea$ is removed from the superscript, then the statement is also not true
$$
w\nVdash [\text{Alice,Bob,Cathy}]^{\varnothing,\{n\}}_{\{yea\}}\,\text{``Approved''}
$$
because the newspaper's information is no longer assumed to be trusted. Without the {\em ex ante} trust in the existence of at least 10 members who are ready to support the regulation, the coalition consisting of Alice, Bob, and Cathy does not have a strategy to pass the regulation. 

Finally, without the {\em ex post} trust that the votes will be counted correctly, such strategy does not exist either:
$$
w\nVdash [\text{Alice,Bob,Cathy}]^{\{yea\},\varnothing}_{\{yea\}}\,\text{``Approved''}.
$$

Note that formula $[C]^{T,T}_X\phi$ is equivalent to our original modality $[C]^{T}_X\phi$.

 The axioms of our logical system could be modified in a straightforward way for this more general modality. The most significant change is for the Public Belief axiom:
$\B^{T}_\varnothing\phi\to[\varnothing]^{\varnothing,T}_\varnothing\phi$.
Although we did not check the details, we believe that our soundness and completeness results can be adjusted for this more general modality. 

\section{Conclusion}\label{conclusion section}

In this paper, we proposed the notion of doxastic strategies. We gave a sound and complete logical system for reasoning about such strategies and outlined a possible extension of this work that separates {\em ex ante} and {\em ex post} trust. 


\section*{Acknowledgment}

Junli Jiang acknowledges the support of the Innovation Research 2035 Pilot Plan of Southwest University (NO. SWUPilotPlan018).

\bibliography{naumov}

\clearpage

\appendix

\addtolength{\oddsidemargin}{2.5cm}
	\addtolength{\evensidemargin}{2.5cm}
	\addtolength{\textwidth}{-5cm}

	\addtolength{\topmargin}{2.5cm}
	\addtolength{\textheight}{-5cm}


\onecolumn

\begin{center}
    {\LARGE\sc Technical Appendix}
\end{center}

\vspace{3mm}
\begin{center}
    {\bf This appendix is not a part of AAAI proceedings.} 
\end{center}

\section{Soundness}\label{soundness section}

In this section, we prove the soundness of our logical system. The soundness of the Truth, the Negative Introspection, the Distributivity, and the Monotonicity axioms is straightforward. Below, we show the soundness of each of the remaining axioms for an arbitrary game  $(W,\{\sim_x\}_{x\in V},\{\mathcal{T}_w\}_{w\in W},\Delta,M,\pi)$ as a separate lemma. We state the strong soundness for the whole system as Theorem~\ref{strong soundness theorem} at the end of this section.

\begin{lemma}[Trust]
$w\Vdash \B^T_X(\B^T_Y\phi\to \phi)$.
\end{lemma}
\begin{proof}
Consider any state $u\in W$ such that $w\sim_X u$ and $T\subseteq\mathcal{T}_u$. By item~4 of Definition~\ref{sat}, it suffices to show that $u\Vdash \B^T_Y\phi\to \phi$. Suppose that $u\Vdash \B^T_Y\phi$. By item~3 of Definition~\ref{sat}, it is enough to prove that $u\Vdash\phi$. 

Note that $u\sim_Y u$ because relation $\sim_Y$ is reflexive. Also, $T\subseteq\mathcal{T}_u$ by the choice of world $u$. Then, the assumption $u\Vdash \B^T_Y\phi$ implies that $u\Vdash\phi$ by item~4 of Definition~\ref{sat}.
\end{proof}

\begin{lemma}[Cooperation]
If $w\Vdash[C]^T_X(\phi\to\psi)$, $w\Vdash[D]^T_X\phi$, and sets $C$ and $D$ are disjoint, then $w\Vdash[C\cup D]^T_X\psi$.
\end{lemma}
\begin{proof}
By item~5 of Definition~\ref{sat}, the assumptions $w\Vdash[C]^T_X(\phi\to\psi)$ and $w\Vdash[D]^T_X\phi$ imply that there are action profiles $s_1\in \Delta^C$ and $s_2\in \Delta^D$ such that 
\begin{eqnarray}
   &&\forall w',v\in W\,\forall\delta\in\Delta^\mathcal{A}(w\sim_X w',T\subseteq\mathcal{T}_{w'},s_1=_C\delta,\nonumber\\
   &&\hspace{10mm}(w',\delta,v)\in M,T\subseteq\mathcal{T}_{v} \Rightarrow v\Vdash\phi\to\psi) \label{apr17-a}
\end{eqnarray}
and
\begin{eqnarray}
   &&\forall w',v\in W\,\forall\delta\in\Delta^\mathcal{A}(w\sim_X w',T\subseteq\mathcal{T}_{w'},s_2=_D\delta,\nonumber\\
   &&\hspace{10mm}(w',\delta,v)\in M,T\subseteq\mathcal{T}_{v} \Rightarrow v\Vdash\phi) \label{apr17-b}
\end{eqnarray}
Consider action profile
\begin{equation}\label{apr17-c}
s(a)=
\begin{cases}
s_1(a), & \mbox{if } a\in C,\\
s_2(a), & \mbox{if } a\in D,\\
\end{cases} 
\end{equation}
of coalition $C\cup D$. Note that profile $s$ is well-defined because sets $C$ and $D$ are disjoint by the assumption of the lemma. Consider any  states $w',v\in W$ and any complete action profile $\delta\in\Delta^\mathcal{A}$ such that $w\sim_{X}w'$, $T\subseteq\mathcal{T}_{w'}$, $s=_{C\cup D} \delta$,  $(w',\delta,v)\in M$, and $T\subseteq\mathcal{T}_{v}$. By item~5 of Definition~\ref{sat}, it suffices to show that $v\Vdash\psi$.

Assumption $s=_{C\cup D} \delta$ implies that $s_1=_{C} \delta$ and $s_2=_{D} \delta$ by equation~(\ref{apr17-c}). Thus, $v\Vdash\phi\to\psi$ and $v\Vdash\phi$ by statements~(\ref{apr17-a}) and (\ref{apr17-b}) and assumptions $w\sim_{X}w'$, $T\subseteq\mathcal{T}_{w'}$,  $(w',\delta,v)\in M$, and $T\subseteq\mathcal{T}_{v}$. Therefore, $v\Vdash\psi$ by item~3 of Definition~\ref{sat}. 
\end{proof}

\begin{lemma}[Strategic Introspection]
$w\Vdash [C]^T_X\phi$ iff $w\Vdash\B^T_X[C]^T_X\phi$. 
\end{lemma}
\begin{proof}
$(\Rightarrow):$
By item~5 of Definition~\ref{sat}, the assumption $w\Vdash[C]^T_X\phi$ implies that there is an action profile $s\in \Delta^C$ of coalition $C$ such that 
\begin{eqnarray}
   &&\forall w',v\in W\,\forall\delta\in\Delta^\mathcal{A}(w\sim_X w',T\subseteq\mathcal{T}_{w'},s=_C\delta,\nonumber\\
   &&\hspace{10mm}(w',\delta,v)\in M,T\subseteq\mathcal{T}_{v} \Rightarrow v\Vdash\phi). \label{apr20-b}
\end{eqnarray}
Consider any state $u\in W$ such that $w\sim_X u$ and $T\subseteq\mathcal{T}_{u}$. By item~4 of Definition~\ref{sat}, it suffices to show that $u\Vdash[C]^T_X\phi$. Indeed, the assumption $w\sim_X u$ and statement~(\ref{apr20-b}) imply that
$$
\forall w',v\in W\,\forall\delta\in\Delta^\mathcal{A}(u\sim_X w',T\subseteq\mathcal{T}_{w'},s=_C\delta,(w',\delta,v)\in M,T\subseteq\mathcal{T}_{v} \Rightarrow v\Vdash\phi). 
$$
Therefore, $u\Vdash[C]^T_X\phi$ by item~5 of Definition~\ref{sat}.

\vspace{1mm}\noindent
$(\Leftarrow):$ We consider the following two cases separately:

\vspace{1mm}
\noindent{\em Case I:} There is at least one state $u\in W$ such that $T\subseteq \mathcal{T}_{u}$ and
\begin{equation}\label{24-aug-a}
    w\sim_X u.
\end{equation}
Hence, by item~4 of Definition~\ref{sat}, the assumption $w\Vdash\B^T_X[C]^T_X\phi$ implies that $u\Vdash [C]^T_X\phi$. Hence, by item~5 of Definition~\ref{sat},
there is an action profile $s\in \Delta^C$ of coalition $C$ such that 
$$
\forall w',v\in W\,\forall\delta\in\Delta^\mathcal{A}(u\sim_X w',T\subseteq\mathcal{T}_{w'},s=_C\delta,(w',\delta,v)\in M,T\subseteq\mathcal{T}_{v} \Rightarrow v\Vdash\phi). 
$$
Thus, by assumption~(\ref{24-aug-a}),
$$
   \forall w',v\in W\,\forall\delta\in\Delta^\mathcal{A}(w\sim_X w',T\subseteq\mathcal{T}_{w'},s=_C\delta,(w',\delta,v)\in M,T\subseteq\mathcal{T}_{v} \Rightarrow v\Vdash\phi). 
$$
Therefore, $w\Vdash[C]^T_X\phi$ by item~5 of Definition~\ref{sat}.

\vspace{1mm}
\noindent{\em Case II:} There are no states $u\in W$ such that $T\subseteq \mathcal{T}_{u}$ and $w\sim_X u$. Thus, $w\Vdash [C]^T_X\phi$ by item~5 of Definition~\ref{sat}.
\end{proof}

\begin{lemma}[Belief in Unavoidability]
If $w\Vdash\B^T_X[\varnothing]^T_Y\phi$, then $w\Vdash[\varnothing]^T_X\phi$. 
\end{lemma}
\begin{proof}
Note that there is only one action profile $s\in\Delta^\varnothing$ of the empty coalition. As a function, this action profile consists of the empty set of pairs. We denote this profile by $s_0$. 

Consider any state $u\in W$ such that $w\sim_X u$ and $T\subseteq \mathcal{T}_u$. By item~5 of Definition~\ref{sat}, it suffices to prove that
\begin{equation}\label{apr20-c}
   \forall v\in W\,\forall\delta\in\Delta^\mathcal{A}(s_0=_\varnothing\delta,(u,\delta,v)\in M,T\subseteq\mathcal{T}_v \Rightarrow v\Vdash\phi). 
\end{equation}
Indeed, by item~4 of Definition~\ref{sat}, the assumptions $w\sim_{X}u$  and $T\subseteq \mathcal{T}_u$ and the assumption $w\Vdash\B^T_X[\varnothing]^T_Y\phi$ of the lemma imply that $u\Vdash[\varnothing]^T_Y\phi$. Hence, by item~5 of Definition~\ref{sat}, there is an action profile $s_1\in\Delta^\varnothing$ of the empty coalition such that
\begin{equation*}
   \forall w',v\in W\,\forall\delta\in\Delta^\mathcal{A}(u\sim_Y w',s_1=_\varnothing\delta,(w',\delta,v)\in M,T\subseteq\mathcal{T}_v \Rightarrow v\Vdash\phi). 
\end{equation*}
Recall that the action profile of the empty coalition is unique. Thus, $s_0=s_1$. Hence, 
\begin{equation*}
   \forall w',v\in W\,\forall\delta\in\Delta^\mathcal{A}(u\sim_Y w',s_0=_\varnothing\delta,(w',\delta,v)\in M,T\subseteq\mathcal{T}_v \Rightarrow v\Vdash\phi). 
\end{equation*}
The last statement in the case $w'=u$ is equivalent to statement~(\ref{apr20-c}).  
\end{proof}

\begin{lemma}[Public Belief]
If $w\Vdash\B^T_\varnothing\phi$, then $w\Vdash[\varnothing]^T_\varnothing\phi$. 
\end{lemma}
\begin{proof}
To prove $w\Vdash[\varnothing]^T_X\phi$, by item~5 of Definition~\ref{sat}, it suffices to show that $v\Vdash\phi$ for each state $v\in W$ such that $T\subseteq\mathcal{T}_v$. Indeed, consider any such state $v\in W$. Note that, vacuously, $w\sim_\varnothing v$. Therefore, $v\Vdash\phi$ by item~4 of Definition~\ref{sat}, the assumption $T\subseteq\mathcal{T}_v$, and the assumption $w\Vdash\B^T_\varnothing\phi$ of the lemma.
\end{proof}

The soundness theorem below follows from the lemmas above. 

\vspace{1mm}
\noindent{\bf Theorem~\ref{strong soundness theorem}.} 
{\em
If $\vdash\phi$, then $w\Vdash\phi$ for each state $w$ of an arbitrary game.
}

\section{Auxiliary Lemmas}

\begin{lemma}[deduction]\label{deduction lemma}
If $F,\phi\vdash\psi$, then $F\vdash\phi\to\psi$.
\end{lemma}
\begin{proof}
Suppose that sequence $\psi_1,\dots,\psi_n$ is a proof from set $F\cup\{\phi\}$ and the theorems of our logical system that uses the Modus Ponens inference rule only. In other words, for each $k\le n$, either
\begin{enumerate}
    \item $\vdash\psi_k$, or
    \item $\psi_k\in F$, or
    \item $\psi_k$ is equal to $\phi$, or
    \item there are $i,j<k$ such that formula $\psi_j$ is equal to $\psi_i\to\psi_k$.
\end{enumerate}
It suffices to show that $F\vdash\phi\to\psi_k$ for each $k\le n$. We prove this by induction on $k$ by considering the four cases above separately.

\vspace{1mm}
\noindent{\bf Case 1}: $\vdash\psi_k$. Note that $\psi_k\to(\phi\to\psi_k)$ is a propositional tautology, and thus, is an axiom of our logical system. Hence, $\vdash\phi\to\psi_k$ by the Modus Ponens inference rule. Therefore, $F\vdash\phi\to\psi_k$. 

\vspace{1mm}
\noindent{\bf Case 2}: $\psi_k\in F$. Note again that $\psi_k\to(\phi\to\psi_k)$ is a propositional tautology, and thus, is an axiom of our logical system. Therefore, by the Modus Ponens inference rule, $F \vdash\phi\to\psi_k$.  

\vspace{1mm}
\noindent{\bf Case 3}: formula $\psi_k$ is equal to $\phi$. Thus, $\phi\to\psi_k$ is a propositional tautology. Therefore, $F\vdash\phi\to\psi_k$. 

\vspace{1mm}
\noindent{\bf Case 4}:  formula $\psi_j$ is equal to $\psi_i\to\psi_k$ for some $i,j<k$. Thus, by the induction hypothesis, $F\vdash\phi\to\psi_i$ and $F\vdash\phi\to(\psi_i\to\psi_k)$. Note that formula 
$
(\phi\to\psi_i)\to((\phi\to(\psi_i\to\psi_k))\to(\phi\to\psi_k))
$
is a propositional tautology. Therefore, $F\vdash \phi\to\psi_k$ by applying the Modus Ponens inference rule twice.
\end{proof}

\begin{lemma}[Lindenbaum]\label{Lindenbaum's lemma}
Any consistent set of formulae can be extended to a maximal consistent set of formulae.
\end{lemma}
\begin{proof}
The standard proof of Lindenbaum's lemma~\cite[Proposition 2.14]{m09} applies here.
\end{proof}

\begin{lemma}\label{S Necessitation rule is deriavable}
Inference rule $\dfrac{\phi}{[\varnothing]^T_X\phi}$ is derivable.
\end{lemma}
\begin{proof}
Suppose $\vdash\phi$. Thus, $\vdash\B^\varnothing_\varnothing\phi$ by the Necessitation inference rule. Hence, $\vdash[\varnothing]^\varnothing_\varnothing\phi$ by the Public Belief axiom and the Modus Ponens inference rule. 
Therefore, $\vdash [\varnothing]^T_X\phi$ by the Monotonicity axiom and the Modus Ponens inference rule.
\end{proof}

\begin{lemma}\label{strategic introspection plus}
    $\vdash [C]^T_X\phi\to \B^\varnothing_X[C]^T_X\phi$.
\end{lemma}
\begin{proof}
The Strategic Introspection axiom, by the laws of propositional reasoning, implies 
$\vdash \B^T_X[C]^T_X\phi\to [C]^T_X\phi$. Thus, 
$\vdash \B^\varnothing_\varnothing(\B^T_X[C]^T_X\phi\to [C]^T_X\phi)$
by the Necessitation inference rule.
Hence,
$\vdash \B^\varnothing_X(\B^T_X[C]^T_X\phi\to [C]^T_X\phi)$
by the Monotonicity axiom and the Modus Ponens inference rule.
Then, by the Distributivity axiom and the Modus Ponens inference rule,
\begin{equation}\label{25-aug-a}
   \vdash \B^\varnothing_X\B^T_X[C]^T_X\phi\to \B^\varnothing_X[C]^T_X\phi. 
\end{equation}
At the same time, by Lemma~\ref{positive introspection lemma},
\begin{equation}\label{25-aug-b}
    \vdash \B^T_X[C]^T_X\phi\to \B^\varnothing_X\B^T_X[C]^T_X\phi.
\end{equation}
Also, by the Strategic Introspection axiom and propositional reasoning,
\begin{equation}\label{25-aug-c}
    \vdash [C]^T_X\phi\to \B^T_X[C]^T_X\phi.
\end{equation}
Finally, $\vdash [C]^T_X\phi\to \B^\varnothing_X[C]^T_X\phi$,
by the laws of propositional reasoning from statements~(\ref{25-aug-a}), (\ref{25-aug-b}), and (\ref{25-aug-c}).
\end{proof}

\noindent{\bf Lemma \ref{super distributivity}.}\,
{\em
If $\phi_1,\dots,\phi_n\vdash\psi$, then $\B^T_X\phi_1,\dots,\B^T_X\phi_n\vdash\B^T_X\psi$.
}

\begin{proof}
By Lemma~\ref{deduction lemma} applied $n$ times, the assumption $\phi_1,\dots,\phi_n\vdash\psi$ implies that
$$
\vdash\phi_1\to(\phi_2\to\dots(\phi_n\to\psi)\dots).
$$
Thus, by the Necessitation inference rule,
$$
\vdash\B^\varnothing_\varnothing(\phi_1\to(\phi_2\to\dots(\phi_n\to\psi)\dots)).
$$
By the Monotonicity axiom and the Modus Ponens inference rule,
$$
\vdash\B^T_X(\phi_1\to(\phi_2\to\dots(\phi_n\to\psi)\dots)).
$$
Hence, by the Distributivity axiom and the Modus Ponens inference rule,
$$
\vdash\B^T_X\phi_1\to\B^T_X(\phi_2\to\dots(\phi_n\to\psi)\dots).
$$
Then, again by the Modus Ponens inference rule,
$$
\B^T_X\phi_1\vdash\B^T_X(\phi_2\to\dots(\phi_n\to\psi)\dots).
$$
Thus, $\B^T_X\phi_1,\dots,\B^T_X\phi_n\vdash\B^T_X\psi$ by applying the previous steps $(n-1)$ more times.
\end{proof}

\begin{lemma}\label{super distributivity 2}
If $\phi_1,\dots,\phi_n\vdash\psi$, then $[\varnothing]^T_X\phi_1,\dots,[\varnothing]^T_X\phi_n\vdash[\varnothing]^T_X\psi$.
\end{lemma}
\begin{proof}
By Lemma~\ref{deduction lemma} applied $n$ times, the assumption $\phi_1,\dots,\phi_n\vdash\psi$ implies that
$$
\vdash\phi_1\to(\phi_2\to\dots(\phi_n\to\psi)\dots).
$$
Thus, by Lemma~\ref{S Necessitation rule is deriavable},
$$
\vdash[\varnothing]^T_X(\phi_1\to(\phi_2\to\dots(\phi_n\to\psi)\dots)).
$$
Hence, by the Cooperation axiom and the Modus Ponens inference rule,
$$
\vdash[\varnothing]^T_X\phi_1\to[\varnothing]^T_X(\phi_2\to\dots(\phi_n\to\psi)\dots).
$$
Then, by the Modus Ponens inference rule,
$$
[\varnothing]^T_X\phi_1\vdash[\varnothing]^T_X(\phi_2\to\dots(\phi_n\to\psi)\dots).
$$
Thus, again by the Cooperation axiom and the Modus Ponens inference rule,
$$
[\varnothing]^T_X\phi_1\vdash[\varnothing]^T_X\phi_2\to[\varnothing]^T_X(\phi_3\to\dots(\phi_n\to\psi)\dots).
$$
Hence, by the Modus Ponens inference rule,
$$
[\varnothing]^T_X\phi_1,[\varnothing]^T_X\phi_2
\vdash
[\varnothing]^T_X(\phi_3\to\dots(\phi_n\to\psi)\dots).
$$
Therefore, $[\varnothing]^T_X\phi_1,\dots,[\varnothing]^T_X\phi_n\vdash[\varnothing]^T_X\psi$ by applying the previous steps $(n-2)$ more times.
\end{proof}

\section{Proof of Lemma~\ref{H child exists}}

The rest of this appendix contains a highly non-trivial proof of Lemma~\ref{H child exists} from the main part of the paper. For the convenience of the reader, we have divided this proof into three subsections. 

\subsection{$T$-Harmony}\label{Harmony section}

We start the proof with an auxiliary notion of {\em harmony}, which has been used in~\cite{nt18ai,jn22ai}. Here, we modify this notion into $T$-harmony.

\begin{definition}\label{harmony}
For any dataset $T\subseteq V$, a pair of sets of formulae $(F,G)$ is in $T$-{\bf\em harmony} if $F\nvdash [\varnothing]^T_V\neg\bigwedge G'$ for each finite set $G'\subseteq G$.
\end{definition}

\begin{lemma}\label{harmony is consistent lemma}
If pair of sets of formulae $(F,G)$ is in $T$-harmony, then sets $F$ and $G$ are consistent.
\end{lemma}
\begin{proof}
If set $F$ is inconsistent, then $F\vdash\phi$ for any formula $\phi\in\Phi$. In particular, $F\vdash [\varnothing]^T_V\neg\bigwedge \varnothing$. Therefore, the pair $(F,G)$ is not in harmony by Definition~\ref{harmony}.

Next, let set $G$ be inconsistent. Thus, by Lemma~\ref{deduction lemma} and propositional reasoning, $\vdash\neg\bigwedge G'$ for some finite set $G'\subseteq G$. Hence,  $\vdash[\varnothing]^T_V\neg\bigwedge G'$ by Lemma~\ref{S Necessitation rule is deriavable}. Thus, $F\vdash[\varnothing]^T_V\neg\bigwedge G'$. Therefore, the pair $(F,G)$ is not in harmony by Definition~\ref{harmony}. 
\end{proof}

Next, we show that certain initial sets are in $T$-harmony.

\begin{lemma}\label{harmony base lemma}
For any consistent set $E\subseteq \Phi$, any formula $\neg[C]^T_X\phi\in E$, any family $\{D_i\}_{i\in I}$ of pairwise disjoint subsets of $C$, and any set $\{[D_i]^T_X\psi_i\}_{i\in I}$ of formulae from set $E$, if 
\begin{eqnarray}
F &=& \{\chi\;|\; \B^\varnothing_X\chi\in E\}\cup\{\B^T_Y\tau\to\tau\;|\;\tau\in\Phi,Y\subseteq V\},\label{april10-F}\\
G &=& \{\neg\phi\}\cup\{\psi_i\}_{i\in I}\cup\{\sigma\;|\;\B^\varnothing_\varnothing\sigma\in E\}\nonumber\\
&&\cup\{\B^T_Z\eta\to\eta\;|\;\eta\in\Phi,Z\subseteq V\},\label{apr9-G}
\end{eqnarray}
then the pair $(F,G)$ is in $T$-harmony.
\end{lemma}
\begin{proof}
Suppose that the pair $(F,G)$ is not in $T$-harmony. Thus, by Definition~\ref{harmony}, there is a finite set $G'\subseteq G$ such that
\begin{equation}\label{apr10-FG'}
    F\vdash [\varnothing]^T_V\neg\bigwedge G'.
\end{equation}
The assumption that set $G'$ is finite and equation~(\ref{apr9-G}) imply that there are formulae 
\begin{equation}\label{apr10-sigma}
\B^\varnothing_\varnothing\sigma_1,\dots\B^\varnothing_\varnothing\sigma_m\in E,
\end{equation}
formulae
$$
\B^T_{Z_1}\eta_1,\dots,\B^T_{Z_k}\eta_k\in \Phi,
$$
and indices $i_1,\dots,i_n\in I$ such that
$$
\{\sigma_i\}_{i\le m},\{\B^T_{Z_j}\eta_j\to\eta_j\}_{j\le k},\{\psi_{i_s}\}_{s\le n},\neg\phi\vdash\bigwedge G'. 
$$
Thus, by Lemma~\ref{deduction lemma} and the laws of propositional reasoning,
$$
\neg\bigwedge G',\{\sigma_i\}_{i\le m},\{\B^T_{Z_j}\eta_j\to\eta_j\}_{j\le k},\{\psi_{i_s}\}_{s\le n}\vdash\phi. 
$$
Again by Lemma~\ref{deduction lemma} and the laws of propositional reasoning,

$$
\neg\bigwedge G',\{\sigma_i\}_{i\le m},\{\B^T_{Z_j}\eta_j\to\eta_j\}_{j\le k}\vdash \psi_{i_1}\to (\psi_{i_2}\to \dots (\psi_{i_n} \to \phi)\dots). 
$$
Thus, by Lemma~\ref{super distributivity 2},
\begin{eqnarray*}
    &&[\varnothing]^T_V\neg\bigwedge G',\{[\varnothing]^T_V\sigma_i\}_{i\le m},\{[\varnothing]^T_V(\B^T_{Z_j}\eta_j\to\eta_j)\}_{j\le k}\\
    &&\hspace{30mm}\vdash [\varnothing]^T_V(\psi_{i_1}\to (\psi_{i_2}\to \dots (\psi_{i_n} \to \phi)\dots)). 
\end{eqnarray*}
By the Monotonicity axiom and the Modus Ponens inference rule applied $m+k$ times,
\begin{eqnarray*}
    &&[\varnothing]^T_V\neg\bigwedge G',\{[\varnothing]^\varnothing_\varnothing\sigma_i\}_{i\le m},\{[\varnothing]^T_\varnothing(\B^T_{Z_j}\eta_j\to\eta_j)\}_{j\le k}\\
    &&\hspace{30mm}\vdash [\varnothing]^T_V(\psi_{i_1}\to (\psi_{i_2}\to \dots (\psi_{i_n} \to \phi)\dots)). 
\end{eqnarray*}
By the Public Belief axiom and the Modus Ponens inference rule applied $m+k$ times,
\begin{eqnarray*}
    &&[\varnothing]^T_V\neg\bigwedge G',\{\B^\varnothing_\varnothing\sigma_i\}_{i\le m},\{\B^T_\varnothing(\B^T_{Z_j}\eta_j\to\eta_j)\}_{j\le k}\\
    &&\hspace{30mm}\vdash [\varnothing]^T_V(\psi_{i_1}\to (\psi_{i_2}\to \dots (\psi_{i_n} \to \phi)\dots)). 
\end{eqnarray*}
By the Trust axiom,
\begin{eqnarray*}
    [\varnothing]^T_V\neg\bigwedge G',\{\B^\varnothing_\varnothing\sigma_i\}_{i\le m}
    \vdash [\varnothing]^T_V(\psi_{i_1}\to (\psi_{i_2}\to \dots (\psi_{i_n} \to \phi)\dots)). 
\end{eqnarray*}
Hence, by statement~(\ref{apr10-FG'}),
\begin{eqnarray*}
    F,\{\B^\varnothing_\varnothing\sigma_i\}_{i\le m}
    \vdash [\varnothing]^T_V(\psi_{i_1}\to (\psi_{i_2}\to \dots (\psi_{i_n} \to \phi)\dots)). 
\end{eqnarray*}
Thus, by equation~(\ref{april10-F}), there are formulae
\begin{equation}\label{apr10-chi}
\B^\varnothing_X\chi_1,\dots,\B^\varnothing_X\chi_p\in E
\end{equation}
and formulae
$$
\B^T_{Y_1}\tau_1,\dots,\B^T_{Y_q}\tau_q
$$
such that
\begin{eqnarray*}
&&\{\chi_t\}_{t\le p},\{\B^T_{Y_r}\tau_r\to\tau_r\}_{r\le q},\,\{\B^\varnothing_\varnothing\sigma_i\}_{i\le m}\\
    &&\hspace{15mm}\vdash [\varnothing]^T_V(\psi_{i_1}\to (\psi_{i_2}\to \dots (\psi_{i_n} \to \phi)\dots)).
\end{eqnarray*}
Then, by Lemma~\ref{super distributivity},
\begin{eqnarray*}
&&\{\B^T_X\chi_t\}_{t\le p},\{\B^T_X(\B^T_{Y_r}\tau_r\to\tau_r)\}_{r\le q},\,\{\B^T_X\B^\varnothing_\varnothing\sigma_i\}_{i\le m}\\
    &&\hspace{15mm}\vdash \B^T_X[\varnothing]^T_V(\psi_{i_1}\to (\psi_{i_2}\to \dots (\psi_{i_n} \to \phi)\dots)).
\end{eqnarray*}
Thus, by the Trust axiom,
\begin{eqnarray*}
\{\B^T_X\chi_t\}_{t\le p},\,\{\B^T_X\B^\varnothing_\varnothing\sigma_i\}_{i\le m}\vdash \B^T_X[\varnothing]^T_V(\psi_{i_1}\to (\psi_{i_2}\to \dots (\psi_{i_n} \to \phi)\dots)).
\end{eqnarray*}
Hence, by the Monotonicity axiom and the Modus Ponens inference rule applied $p+m$ times,
\begin{eqnarray*}
\{\B^\varnothing_X\chi_t\}_{t\le p},\,\{\B^\varnothing_\varnothing\B^\varnothing_\varnothing\sigma_i\}_{i\le m}\vdash \B^T_X[\varnothing]^T_V(\psi_{i_1}\to (\psi_{i_2}\to \dots (\psi_{i_n} \to \phi)\dots)).
\end{eqnarray*}
Then, by Lemma~\ref{positive introspection lemma} and the Modus Ponens inference rule applied $m$ times,
\begin{eqnarray*}
\{\B^\varnothing_X\chi_t\}_{t\le p},\,\{\B^\varnothing_\varnothing\sigma_i\}_{i\le m}\vdash \B^T_X[\varnothing]^T_V(\psi_{i_1}\to (\psi_{i_2}\to \dots (\psi_{i_n} \to \phi)\dots)).
\end{eqnarray*}
Thus, by statements~(\ref{apr10-sigma}) and (\ref{apr10-chi}),
\begin{eqnarray*}
E\vdash \B^T_X[\varnothing]^T_V(\psi_{i_1}\to (\psi_{i_2}\to \dots (\psi_{i_n} \to \phi)\dots)).
\end{eqnarray*}
Hence, by the Belief in Unavoidability axiom and the Modus Ponens inference rule,
\begin{eqnarray*}
E\vdash [\varnothing]^T_X(\psi_{i_1}\to (\psi_{i_2}\to \dots (\psi_{i_n} \to \phi)\dots)).
\end{eqnarray*}
Then, by the Cooperation axiom and the Modus Ponens inference rule,
\begin{eqnarray*}
E\vdash [D_{i_1}]^T_X\psi_{i_1}\to [D_{i_1}]^T_X(\psi_{i_2}\to \dots (\psi_{i_n} \to \phi)\dots).
\end{eqnarray*}
Thus, by the assumption $[D_{i_1}]^T_X\psi_{i_1}\in E$ of the lemma and the Modus Ponens inference rule,
\begin{eqnarray*}
E\vdash [D_{i_1}]^T_X(\psi_{i_2}\to \dots (\psi_{i_n} \to \phi)\dots).
\end{eqnarray*}
Hence, by the Cooperation axiom, the Modus Ponens inference rule,
and the assumption of the lemma that
family $\{D_i\}_{i\in I}$ consists of pairwise disjoint sets,
\begin{eqnarray*}
E\vdash [D_{i_2}]^T_X\psi_{i_2}\to [D_{i_1}\cup D_{i_2}]^T_X(\psi_{i_3}\to\dots (\psi_{i_n} \to \phi)\dots).
\end{eqnarray*}
Then, by the assumption $[D_{i_2}]^T_X\psi_{i_2}\in E$ of the lemma and the Modus Ponens inference rule,
\begin{eqnarray*}
E\vdash [D_{i_1}\cup D_{i_2}]^T_X(\psi_{i_3}\to\dots (\psi_{i_n} \to \phi)\dots)).
\end{eqnarray*}
Thus, by repeating the last step $n-2$ more times,
\begin{eqnarray*}
E\vdash [D_{i_1}\cup\dots \cup D_{i_n}]^T_X\phi.
\end{eqnarray*}
Hence, 
$E\vdash [C]^T_X\phi$
by the Monotonicity axiom, the Modus Ponens inference rule, and the assumption of the lemma that family $\{D_i\}_{i\in I}$ are subsets of $C$. 
Then, $\neg [C]^T_X\phi\notin E$ because set $E$ is consistent, which contradicts the assumption $\neg [C]^T_X\phi\in E$ of the lemma.
\end{proof}

The next lemma shows that any pair in harmony can be extended.

\begin{lemma}\label{harmony step lemma}
For any pair of sets of formulae $(F,G)$ in $T$-harmony, any formula $[\varnothing]^T_X\phi\in\Phi$, either the pair 
$(F\cup\{\neg[\varnothing]^T_X\phi\},G)$ or the pair $(F,G\cup\{\phi\})$ is in $T$-harmony.
\end{lemma}
\begin{proof}
Suppose that neither the pair $(F\cup\{\neg[\varnothing]^T_X\phi\},G)$ nor the pair $(F,G\cup\{\phi\})$ is in $T$-harmony. Thus, by Definition~\ref{harmony}, 
\begin{eqnarray}
&&F, \neg[\varnothing]^T_X\phi\vdash [\varnothing]^T_V\neg\bigwedge G_1\label{apr13-a},\\
&&F\vdash [\varnothing]^T_V\neg\bigwedge G_2,\label{apr13-b}
\end{eqnarray}
for some finite sets $G_1\subseteq G$ and $G_2\subseteq G\cup\{\phi\}$. Then, there must exist a finite set of formulae $G'\subseteq G$ such that,
\begin{eqnarray*}
&&\vdash \bigwedge G'\to \bigwedge G_1,\\
&&\vdash \phi\to\left(\bigwedge G'\to \bigwedge G_2\right).
\end{eqnarray*}
Hence, by the laws of propositional reasoning,
\begin{eqnarray*}
&&\vdash \neg\bigwedge G_1\to \neg\bigwedge G',\\
&&\vdash \neg\bigwedge G_2\to\left(\phi\to \neg\bigwedge G'\right).
\end{eqnarray*}
Thus, by Lemma~\ref{S Necessitation rule is deriavable},
\begin{eqnarray*}
&&\vdash [\varnothing]^T_V\left(\neg\bigwedge G_1\to \neg\bigwedge G'\right),\\
&&\vdash [\varnothing]^T_V\left(\neg\bigwedge G_2\to\left(\phi\to \neg\bigwedge G'\right)\right).
\end{eqnarray*}
Then, by the Cooperation axiom and the Modus Ponens inference rule,
\begin{eqnarray*}
&&\vdash [\varnothing]^T_V\left(\neg\bigwedge G_1\right)\to [\varnothing]^T_V\left(\neg\bigwedge G'\right),\\
&&\vdash [\varnothing]^T_V\left(\neg\bigwedge G_2\right)\to[\varnothing]^T_V\left(\phi\to \neg\bigwedge G'\right).
\end{eqnarray*}
Hence, by assumptions~(\ref{apr13-a}) and~(\ref{apr13-b}) and the Modus Ponens inference rule,
\begin{eqnarray*}
&&F, \neg[\varnothing]^T_X\phi\vdash [\varnothing]^T_V\left(\neg\bigwedge G'\right),\label{3-feb-a}\\
&&F\vdash[\varnothing]^T_V\left(\phi\to \neg\bigwedge G'\right).\label{3-feb-b}
\end{eqnarray*}
Hence, by the contraposition of the Monotonicity axiom,
\begin{eqnarray*}
&&F, \neg[\varnothing]^T_V\phi\vdash [\varnothing]^T_V\left(\neg\bigwedge G'\right),\\
&&F\vdash[\varnothing]^T_V\left(\phi\to \neg\bigwedge G'\right).\\
\end{eqnarray*}
Then, by the Cooperation axiom and the Modus Ponens inference rule,
\begin{eqnarray*}
&&F, \neg[\varnothing]^T_V\phi\vdash [\varnothing]^T_V\left(\neg\bigwedge G'\right),\\
&&F\vdash[\varnothing]^T_V\phi\to [\varnothing]^T_V\left(\neg\bigwedge G'\right).
\end{eqnarray*}
Hence, by Lemma~\ref{deduction lemma},
\begin{eqnarray*}
&&F\vdash \neg[\varnothing]^T_V\phi\to[\varnothing]^T_V\left(\neg\bigwedge G'\right),\\
&&F\vdash[\varnothing]^T_V\phi\to [\varnothing]^T_V\left(\neg\bigwedge G'\right).
\end{eqnarray*}
Then, by the laws of propositional reasoning,
$$
F\vdash [\varnothing]^T_V\left(\neg\bigwedge G'\right).
$$
Therefore, by Definition~\ref{harmony} and the assumption $G'\subseteq G$, the pair $(F,G)$ is not in $T$-harmony, which contradicts the assumption of the lemma that it is in $T$-harmony.
\end{proof}

\begin{definition}\label{complete harmony definition}
A pair of sets of formulae $(F,G)$ is in complete $T$-harmony if this pair is in $T$-harmony and, for any formula $[\varnothing]^T_X\phi\in\Phi$, either $\neg[\varnothing]^T_X\phi\in F$ or $\phi\in G$. 
\end{definition}

The next lemma follows from Lemma~\ref{harmony step lemma} and Definition~\ref{complete harmony definition}. 

\begin{lemma}\label{complete harmony lemma}
For any pair $(F,G)$ in $T$-harmony, there is a pair $(F',G')$ in complete $T$-harmony such that $F\subseteq F'$, $G\subseteq G'$. \qed
\end{lemma}

\subsection{Cones}

\begin{definition}
For any state $w\in W$, let $Cone(w)$ be the set of all states $u\in W$ such that either $w=u$ or sequence $w$ is a prefix of sequence $u$.
\end{definition}

\begin{figure}
\begin{center}
\scalebox{.5}{\includegraphics{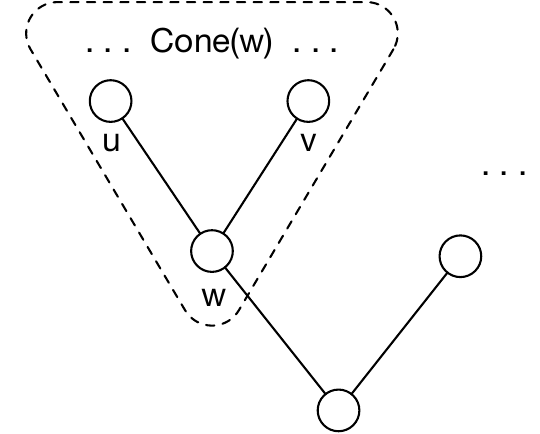}}
\caption{States $w$, $u$, and $v$ belong to $Cone(w)$.}\label{Canonical subtree figure}
\end{center}
\end{figure}

We call set $Cone(w)$ the {\em cone} of a state $w\in W$. Informally, $Cone(w)$ is a subtree of the canonical tree that starts at node $w$, see Figure~\ref{Canonical subtree figure}.
The next two lemmas state an auxiliary property of cones used later. 

\begin{lemma}\label{trees are disjoint}
$Cone(w)\cap Cone(w')=\varnothing$ for any two states $w,w'\in W$ such that 
\begin{eqnarray*}
w&=&T_0,F_0,X_1,b_1,\dots,X_{n-1},b_{n-1},T_{n-1},F_{n-1},X_{n},b_n,T_n,F_{n},\\
w'&=&T_0,F_0,X_1,b_1,\dots,X_{n-1},b_{n-1},T_{n-1},F_{n-1},X'_{n},b'_n,T'_{n},F'_{n},
\end{eqnarray*}
and $(X_{n},b_n,T_n,F_{n})\neq (X'_{n},b'_n,T'_n,F'_{n})$. 
\end{lemma}

\begin{lemma}\label{canonical tree path lemma}
For any states $w,u\in W$ and any data variable $y\in V$, if 
$$w=T_0,F_0,X_1,b_1,\dots,X_{n-1},b_{n-1},T_{n-1},F_{n-1},X_{n},b_n,T_n,F_{n}\in W,$$
$u\notin Cone(w)$, and $w\sim_y u$, then $y\in X_n$.
\end{lemma}
\begin{proof} 
Let $w_0=T_0,F_0,X_1,b_1,\dots,X_{n-1},b_{n-1},T_{n-1},F_{n-1}$.
Consider simple path between vertices $w$ and $u$. The assumption $u\notin Cone(w)$ of the lemma implies that this path must contain edge $(w_0,w)$. Hence, by the assumption $w\sim_y u$ of the lemma and Definition~\ref{canonical sim}, edge $(w_0,w)$ is labeled with variable $y$. Hence, $y\in X_n$.
\end{proof}

\subsection{Final Steps}

\begin{lemma}\label{pigeonhole lemma}
For any set $A$ and any pairwise disjoint family of sets $\{R_b\}_{b\in B}$, if the cardinality of set $A$ is less than the cardinality of set $B$, then there is an index $b\in B$ such that $A\cap R_b=\varnothing$. \qed
\end{lemma}

We are now ready to finish the proof of Lemma~\ref{H child exists}.

\noindent{\bf Lemma~\ref{H child exists}} {\em
For an arbitrary state $w\in W$, any formula $\neg[C]^T_X\phi\in F(w)$, and any action profile $s\in\Delta^C$, there are states $w',v\in W$ and a complete action profile $\delta\in\Delta^\mathcal{A}$ such that $w\sim_{X} w'$, $s=_C\delta$, $T\subseteq \mathcal{T}_{w'}$, $T\subseteq \mathcal{T}_{v}$, $(w',\delta,v)\in M$, and $\phi\notin F(v)$. 
}

\begin{proof}
Consider any state $w\in W$, any formula $\neg[C]^T_X\phi\in F(w)$, and any action profile $s\in\Delta^C$. For each formula $\psi\in\Phi$, define coalition
\begin{equation}\label{apr23-c}
    D_\psi=\{a\in C\;|\;pr_1(s(a))=\psi\}.
\end{equation}
\begin{claim}\label{apr21-c}
$\{D_\psi\}_{\psi\in\Phi}$ is a family of pairwise disjoint subsets of $C$.
\end{claim}
\begin{proof-of-claim}
By equation~(\ref{apr23-c}), set $D_\psi$ is a subset of $C$ for each formula $\psi\in\Phi$. To show that sets $\{D_\psi\}_{\psi\in\Phi}$ are disjoint, suppose that there is an actor $a\in C$ such that $a\in D_{\psi_1}\cap D_{\psi_2}$ for some formulae $\psi_1,\psi_2\in\Phi$. Therefore, $\psi_1=pr_1(s(a))=\psi_2$ by equation~(\ref{apr23-c}).
\end{proof-of-claim}

\vspace{1mm}
Consider the set of formulae
\begin{equation}\label{apr21-d}
    \Psi=\{\psi\in \Phi\;|\;[D_\psi]^T_{X}\psi\in F(w)\}.
\end{equation}

Next, we are going to apply Lemma~\ref{harmony base lemma}. Note that $F(w)$ is a consistent subset of $\Phi$ and that $\neg[C]^T_X\phi\in F(w)$ by the assumption of the lemma. Also, observe that $\{D_\psi\}_{\psi\in\Psi}\subseteq \{D_\psi\}_{\psi\in\Phi}$ is a family of pairwise disjoint subsets of $C$ by Claim~\ref{apr21-c}. Finally, note that  $\{[D_\psi]^T_{X}\psi\}_{\psi\in\Psi}\subseteq F(w)$ by equation~(\ref{apr21-d}). Therefore, by Lemma~\ref{harmony base lemma}, the pair $(F,G)$ is in $T$-harmony, where
\begin{eqnarray}
F &=& \{\chi\;|\; \B^\varnothing_{X}\chi\in F(w)\}\cup\{\B^T_Y\tau\to\tau\;|\;\tau\in\Phi, Y\subseteq V\},\label{apr23-d}\\
G &=& \{\neg\phi\}\cup\Psi\cup\{\sigma\;|\;\B^\varnothing_\varnothing\sigma\in F(w)\}\nonumber\\
&&\cup\{\B^T_Z\eta\to\eta\;|\;\eta\in\Phi,Z\subseteq V\}
.\label{apr23-a}
\end{eqnarray}
Then, by Lemma~\ref{complete harmony lemma}, there is a pair $(F',G')$ in complete $T$-harmony such that $F\subseteq F'$ and $G\subseteq G'$. Hence, by Definition~\ref{complete harmony definition}, pair $(F',G')$ in $T$-harmony. Thus, by Lemma~\ref{harmony is consistent lemma}, sets $F'$ and $G'$ are consistent. Then, by Lemma~\ref{Lindenbaum's lemma}, sets $F'$ and $G'$ can be further extended to maximal consistent sets of formulae $F''$ and $G''$, respectively, such that 
\begin{equation}\label{apr21-e}
    F\subseteq F'\subseteq F'' \;\;\;\mbox{and}\;\;\; G\subseteq G'\subseteq G''.
\end{equation}

Let state $w\in W$ be the sequence $T_0,F_0,X_1,b_1,\dots,X_{n},b_n,T_{n},F_{n}$.
For each $b\in \mathcal{B}$, define sequence,
\begin{equation}\label{apr23-e}
    w_b=T_0,F_0,X_1,b_1,\dots,X_{n},b_n,T_{n},F_{n},X,b,T,F''.
\end{equation}
Recall that we previously referred to such nodes as ``clones''. 

\begin{claim}\label{apr24-c}
$w_b\in W$, $w\sim_{X} w_b$, and $T=\mathcal{T}_{w_b}$ for each $b\in\mathcal{B}$.
\end{claim}
\begin{proof-of-claim}
By Definition~\ref{canonical W} and the assumption $w\in W$ of the lemma, to show that $w_b\in W$, it suffices to show that $\{\chi\;|\;\B^\varnothing_{X}\chi\in F_n\}\subseteq F''$  and $\{\B^T_Y\tau\to\tau\;|\;Y\subseteq V,\tau\in F_n\}\subseteq F''$. Both of these statements are true by equation~(\ref{apr23-d}), part $F\subseteq F''$ of statement~(\ref{apr21-e}), and because $F(w)=F_n$.
Statement $T=\mathcal{T}_{w_b}$ follows from Definition~\ref{canonical T} and equation~(\ref{apr23-e}). 
Finally, note that $w\sim_{X} w_b$ for each $b\in\mathcal{B}$ by Definition~\ref{canonical sim} and equation~(\ref{apr23-e}).
\end{proof-of-claim}

\vspace{1mm}

To continue the proof of the lemma,
recall that set $\mathcal{B}$ has been assumed to have larger cardinality than set $\mathcal{A}$. Thus, because $C\subseteq \mathcal{A}$,
$$|\{pr_2(s(a))\;|\;a\in C\}|\le |C|\le |\mathcal{A}|< |\mathcal{B}|=|\{Cone(w_b)\;|\;b\in\mathcal{B}\}|.$$
Note that sets $\{Cone(w_b)\;|\;b\in\mathcal{B}\}$ are pairwise disjoint by Lemma~\ref{trees are disjoint}.
Hence, by Lemma~\ref{pigeonhole lemma}, there must exist at least one element $\beta\in \mathcal{B}$ such that 
\begin{equation}\label{apr23-h}
\{pr_2(s(a))\;|\;a\in C\}\cap Cone(w_{\beta})=\varnothing.
\end{equation}
Let 
\begin{equation}\label{apr23-g}
    w'=w_{\beta}.
\end{equation}
Choose an arbitrary element $\beta'\in\mathcal{B}$ and define state $v$ as follows, see Figure~\ref{two states figure},
\begin{equation}\label{apr23-b}
    v=T_0,F_0,X_1,b_1,\dots,T_{n-1},F_{n-1},X_{n},b_n,T_n,F_{n},\varnothing,\beta',T,G''.
\end{equation}
\begin{claim}\label{apr22-a}
$v\in W$.
\end{claim}
\begin{proof-of-claim}
By Definition~\ref{canonical W} and the assumption $w\in W$ of the lemma, it suffices to show that $\{\sigma\;|\;\B^\varnothing_{\varnothing}\sigma\in F_n\}\cup\{\B^T_Z\eta\to\eta\;|\;\eta\in\Phi,Z\subseteq V\}\subseteq G''$. The latter is true by equation~(\ref{apr23-a}), part $G\subseteq G''$ of statement~(\ref{apr21-e}), and because $F(w)=F_n$.
\end{proof-of-claim}
\vspace{1mm}

\begin{figure}
\begin{center}
\scalebox{.55}{\includegraphics{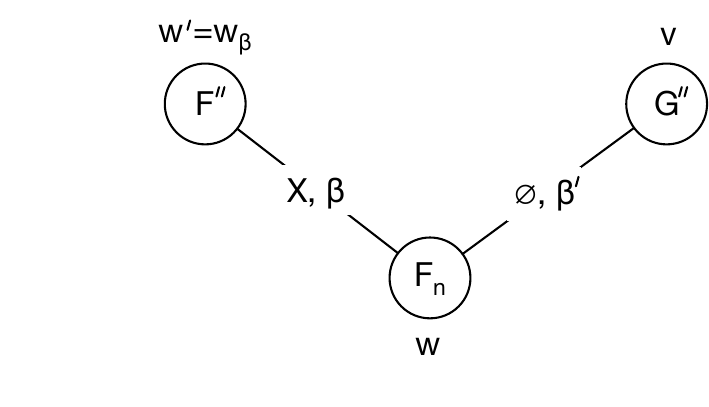}}
\caption{Towards the proof of Lemma~\ref{H child exists}.}\label{two states figure}
\end{center}
\end{figure}

Recall that we started the proof of the lemma by fixing a state $w\in W$ and an action profile $s\in\Delta^C$ of coalition $C$. Define complete action profile $\delta\in\Delta^\mathcal{A}$ as follows: 
\begin{equation}\label{apr24-a}
\delta(a)=
\begin{cases}
s(a), & \mbox{if } a\in C,\\
(\top,w), & \mbox{otherwise}.
\end{cases} 
\end{equation}
Note 
\begin{equation}\label{oct27-b}
    s=_C\delta.
\end{equation}




\begin{claim}\label{oct27-a}
$(w',\delta,v)\in M$.
\end{claim}
\begin{proof-of-claim}
Consider any formula 
\begin{equation}\label{apr23-f}
    [D]^Q_Y\psi\in F(w')
\end{equation}
such that
\begin{equation}\label{apr24-b}
    \forall a\in D\,(pr_1(\delta(a))=\psi),
\end{equation}
\begin{equation}\label{apr27-a}
    \forall a\in D\,(w'\sim_Y pr_2(\delta(a))),
\end{equation}
\begin{equation}\label{3-aug-a}
    Q\subseteq \mathcal{T}_{w'},
\end{equation}
and
\begin{equation}
    Q\subseteq \mathcal{T}_{v}.
\end{equation}
By Definition~\ref{canonical mechanism}, it suffices to show that $\psi\in F(v)$. Indeed, first, note that 
\begin{equation}\label{4-aug-a}
   Q\subseteq \mathcal{T}_{w'}=\mathcal{T}_{w_\beta}=T 
\end{equation}
by statement~(\ref{3-aug-a}), equation~(\ref{apr23-g}), and Claim~\ref{apr24-c}. Next, we consider the following three cases separately:

\vspace{1mm}\noindent
{\em Case I:} $D=\varnothing$. Assumption~(\ref{apr23-f}) implies that $[\varnothing]^Q_Y\psi\in F(w_{\beta})$ by equation~(\ref{apr23-g}). Hence,  $F(w_{\beta})\vdash [\varnothing]^T_Y\psi$ by statement~(\ref{4-aug-a}) and the Monotonicity axiom. 
Thus, 
$\neg[\varnothing]^T_Y\psi\notin F(w_{\beta})$
because set $F(w_{\beta})$ is consistent.
Hence, $\neg[\varnothing]^T_Y\psi\notin F''$ by equation~(\ref{apr23-e}). 
Then, $\neg[\varnothing]^T_Y\psi\notin F'$ by part $F'\subseteq F''$ of statement~(\ref{apr21-e}). 
Thus, $\psi\in G'$ by the assumption that the pair $(F',G')$ is in complete $T$-harmony and Definition~\ref{complete harmony definition}. 
Hence, $\psi\in G''$ by part $G'\subseteq G''$ of statement~(\ref{apr21-e}). Therefore, $\psi\in F(v)$ by equation~(\ref{apr23-b}).

\vspace{1mm}\noindent
{\em Case II:} there exists an actor $a\in D\setminus C$. Then,
$\delta(a)=(\top,w)$ by equation~(\ref{apr24-a}). Hence, $pr_1(\delta(a))=\top$. Thus, $\psi=\top$ by equation~(\ref{apr24-b}).
Therefore, $\psi\in F(v)$ because $F(v)$ is a maximal consistent set of formulae.

\vspace{1mm}\noindent
{\em Case III:} $\varnothing\neq D\subseteq C$. We further split this case into two subcases:

\begin{quote}
\noindent
{\em Subcase IIIa:} $Y\subseteq X$. 
Note that
$pr_1(s(a))=pr_1(\delta(a))$ for each actor $a\in D$ by equation (\ref{apr24-a}) and the assumption $D\subseteq C$ of Case III. 
Thus, $pr_1(s(a))=\psi$ for each actor $a\in D$  by statement~(\ref{apr24-b}). Hence, $D\subseteq D_\psi$ by equation (\ref{apr23-c}). Then, by assumption~(\ref{apr23-f}), the assumption $Y\subseteq X$ of the subcase, statement~(\ref{4-aug-a}), the Monotonicity axiom, and the Modus Ponens inference rule,
\begin{equation}\label{apr24-d}
     F(w')\vdash[D_\psi]^T_X\psi.
\end{equation}
Thus, $F(w')\vdash \B^\varnothing_{X}[D_\psi]^T_{X}\psi$ by Lemma~\ref{strategic introspection plus} and the Modus Ponens inference rule. Hence,
\begin{equation}\label{apr24-e}
   \B^\varnothing_{X}[D_\psi]^T_{X}\psi\in F(w'),
\end{equation}
because set $F(w')$ is maximal.
Note that $w\sim_X w_{\beta}$ by Claim~\ref{apr24-c}. Then,
$w\sim_X w'$ by equation~(\ref{apr23-g}). 
Thus, 
$[D_\psi]^T_{X}\psi\in F(w)$ by statement~(\ref{apr24-e}) and Lemma~\ref{child all}. Hence,
$\psi\in\Psi$ by equation~(\ref{apr21-d}). Then,
$\psi\in G$ by equation~(\ref{apr23-a}). Thus,
$\psi\in G''$ by equation~(\ref{apr21-e}). Therefore,
$\psi\in F(v)$ by equation~(\ref{apr23-b}).

\vspace{2mm}\noindent
{\em Subcase IIIb:} there is a data variable $y\in Y\setminus X$.
The assumption $\varnothing\neq D$ of Case III implies that there is an actor $a\in D$. Then, by the assumption $D\subseteq C$ of Case III,
\begin{equation}\label{apr27-c}
    a\in C.
\end{equation}
Also, the assumption $a\in D$ implies $w'\sim_y pr_2(\delta(a))$ by statement~(\ref{apr27-a}) and the assumption $y\in Y$ of Subcase IIIb.
Thus, $w'\sim_y pr_2(s(a))$ by equation~(\ref{apr24-a}) and statement~(\ref{apr27-c}). Hence, by equation~(\ref{apr23-g}),
\begin{equation}\label{apr27-b}
    w_{\beta}\sim_y pr_2(s(a)).
\end{equation}
At the same time, $pr_2(s(a))\notin Cone(w_{\beta})$ by equation~(\ref{apr23-h}) and statement~(\ref{apr27-c}). Therefore, $y\in X$ by Lemma~\ref{canonical tree path lemma}, statement~(\ref{apr27-b}), and equation~(\ref{apr23-e}), which contradicts the assumption $y\in Y\setminus X$ of Subcase IIIb.  
\end{quote}
This concludes the proof of the claim $(w',\delta,v)\in M$.
\end{proof-of-claim}

Thus, towards the proof of the lemma, we constructed state $w'\in W$ (equation~(\ref{apr23-g}) and Claim~\ref{apr24-c}), state $v\in W$ (Claim~\ref{apr22-a}) and a complete action profile $\delta\in\Delta^\mathcal{A}$ such that $w\sim_X w'$ (equations~(\ref{apr23-e}) and (\ref{apr23-g})), $s=_C\delta$ (equation~(\ref{oct27-b})), 
$T\subseteq \mathcal{T}_{w'}$ (equation (\ref{apr23-g}) and Claim~\ref{apr24-c}),
$T\subseteq \mathcal{T}_{v}$ (equation~(\ref{apr23-b})),
and $(w',\delta,v)\in M$ (Claim~\ref{oct27-a}).
To finish the proof of the lemma, note that $\neg\phi\in G\subseteq G'' = F(v)$ by equation~(\ref{apr23-a}), equation~(\ref{apr21-e}), and equation~(\ref{apr23-b}) respectively.
\end{proof}

\end{document}